\documentclass[11pt]{article}
\usepackage{fullpage}

\usepackage{amsmath,amsfonts,bm}

\def\eqref#1{equation~\ref{#1}}
\def\1{\bm{1}}

\DeclareMathAlphabet{\mathsfit}{\encodingdefault}{\sfdefault}{m}{sl}
\SetMathAlphabet{\mathsfit}{bold}{\encodingdefault}{\sfdefault}{bx}{n}

\usepackage{enumerate}
\usepackage{enumitem}

\usepackage{graphicx} % more modern
\usepackage{caption}
\usepackage{subcaption}
\usepackage{amsmath}
\usepackage{amsthm}
\usepackage{amssymb}
\usepackage{tikz}
\usepackage{xcolor}
\usetikzlibrary{arrows}

\usepackage{natbib}

\allowdisplaybreaks

\usepackage{mathrsfs}

\usepackage{algorithm}
\usepackage[noend]{algpseudocode}
\usepackage{hyperref}
\usepackage{bm,todonotes}

\allowdisplaybreaks

\newcommand{\supp}{\mathrm{supp}}

\def\RR{\mathbb{R}}

\def\EE{\mathbb{E}}

\def\cC{\mathcal{C}}

\def\cF{\mathcal{F}}

\def\cX{\mathcal{X}}
\def\cU{\mathcal{U}}

\newcommand{\wt}{\widetilde}
\newcommand{\mat}[1]{#1}

\newcommand{\norm}[1]{\left\|#1\right\|}
\newcommand{\normop}[1]{\left\|#1\right\|_{\mathrm{op}}}

\newtheorem{thm}{Theorem}[section]
\newtheorem{lem}{Lemma}[section]

\newtheorem{prop}{Proposition}[section]
\newtheorem{asmp}{Assumption}[section]
\newtheorem{defn}{Definition}[section]

\newtheorem{rem}{Remark}[section]

\newcommand{\one}{\mathbb{I}}

\newcommand{\metacontroller}{decoder}

\newenvironment{itemize*}%
{\begin{itemize}[leftmargin=*,topsep=0pt]%
		\setlength{\itemsep}{0pt}%
		\setlength{\parskip}{0pt}}%
	{\end{itemize}}
\newenvironment{enumerate*}%
{\begin{enumerate}[leftmargin=*,topsep=0pt]%
		\setlength{\itemsep}{0pt}%
		\setlength{\parskip}{0pt}}%
	{\end{enumerate}}

\usepackage{hyperref}
\usepackage{url}

\title{Continuous Control with Contexts, Provably}

\author{}

\begin{document}
\author{Simon S. Du \\
Institute for Advanced Study \\
\texttt{ssdu@ias.edu }
\and Ruosong Wang \\
Carnegie Mellon University \\
\texttt{ruosongw@andrew.cmu.edu}
\and Mengdi Wang\\
Princeton University\\
\texttt{mengdiw@princeton.edu}
\and Lin F. Yang \\
University of California, Los Angeles \\
\texttt{linyang@ee.ucla.edu}
}
\date{}

\maketitle

\begin{abstract}
A fundamental challenge in artificial intelligence is to build an agent that generalizes and adapts to \emph{unseen} environments. A common strategy is to build a \metacontroller~that takes the context of the unseen new environment as input and generates a policy accordingly. The current paper studies how to build a \metacontroller~for the fundamental continuous control task, linear quadratic regulator (LQR), which can model a wide range of real-world physical environments. We present a simple algorithm for this problem, which uses upper confidence bound (UCB) to refine the estimate of the \metacontroller~and balance the exploration-exploitation trade-off. Theoretically, our algorithm enjoys a $\widetilde{O}\left(\sqrt{T}\right)$ regret bound in the online setting where $T$ is the number of environments the agent played. This also implies after playing   $\widetilde{O}\left(1/\epsilon^2\right)$ environments, the agent is able to transfer the learned knowledge to obtain an $\epsilon$-suboptimal policy for an unseen environment. To our knowledge, this is first provably efficient algorithm to build a \metacontroller~in the continuous control setting. While our main focus is theoretical, we also present experiments that demonstrate the effectiveness of our algorithm.

\end{abstract}

\section{Introduction}
\label{sec:intro}
%bullshit
Humans are able to solve a new task \emph{without any training} based on previous experience in similar tasks.
The desired intelligent agent should be able do the same, learning from previous experience, adapting to the new ones and improving the performance as the agent gains more experience.
This is a challenging problem as we need to design an adaptation mechanism which is fundamentally different from classical supervised learning methods.

%\metacontroller
A common approach is to build a \metacontroller~so that once the agent sees a description of new task, i.e., the context of the new task, 
the \metacontroller~turns the context into a succinct representation of the new task, based on which the agent is able to design a policy to solve the task.
Note this procedure resembles how a human solves a new task.
For example, if a human wants to push an object on a table, the human first sees the object and the table (context).
Then, in his/her mind, the context becomes a representation of this task, e.g., a sense of weight of the object.
Based on this representation, the human can easily reason about how much force to exert on the object.
%\simon{@Ruosong: change the example based what what experiment we do.}

%empirical work
This general approach has been applied in practice.
For example, \citet{wu2018learning} studied the visual navigation task and built a Bayesian model that takes the context of new environments and outputs the policy that enables the agent to navigate.
\citet{killian2016transfer} used this approach to develop personalized medicine policies for HIV treatment.

%existing theoretical and its limitations and why we study LQR
While this is a promising  approach, currently we only have limited theoretical understanding.
The approach can be formulated in Contextual Markov Decision Process (CMDP) framework~\citep{hallak2015contextual}.
Recently, there is a line of work gave provable guarantees for CMDP~\citep{abbasi2014online,hallak2015contextual,dann2018policy,modi2018markov,modi2019contextual}.
These work all study tabular MDPs, and use function approximation, e.g., linear functions, generalized linear models, etc,  to model the mapping from the context to the probability transition matrix.
A major drawback of these work is that they are restricted to the tabular setting and thus can only deal with discrete environments.
Therefore, they can hardly model real-world continuous control tasks, like the task of pushsing an object as we described above.
%\simon{To change according to experiments}
A natural question arises:
\begin{center}
	\textbf{Can we design a provably efficient \metacontroller~for continuous control problems?}
\end{center}

%our contribution
In this paper, we make an important step towards answering this question.
We study the fundamental task in continuous control, linear quadratic regulator (LQR).
LQR is arguably the most widely used framework in continuous control, as LQR easily models real-world physical phenomena, e.g., the pushing object task we described earlier.
%\simon{change example.}
We propose a new algorithm that builds a \metacontroller, so that for a new LQR task, the \metacontroller~takes LQR's context and outputs a representation based on which the agent can infer a near-optimal policy for new continuous control tasks.
%\simon{@Lin: add short algorithm description.}
%\lin{see below.}
In the training phase, we build the \metacontroller~via a sequence of LQRs (in an online fashion) with unknown parameters.
% but their contexts are sampled from some unknown distribution and can be observed.
%The \metacontroller maintains an estimator for some meta-representation of the unknown LQR distribution. 
For each new task, we first use the current \metacontroller~to build the representation of this task, infer a policy based on this representation and use this policy to do control for this episode.
There are two crucial components in our algorithm.
First, after each episode, we will refine the estimate of the \metacontroller~based on the observations from this episode.
Second, it is crucial to use a upper confidence bound (UCB) estimator of the \metacontroller~to build the representation so that the agent can perform a near-optimal trade-off between exploration and exploitation.
% then run a reinforcement learning algorithm that rolls out an episode 
%For each training context, we run a reinforcement learning algorithm that rolls out an episode of data from the LQR using the current best upper-confidence estimator (UCB) of the \metacontroller
%After this, the \metacontroller updates the meta-representation using the newly collected episode.
%The UCB estimator is important since it always maintains a lower bound of the true value function, which guarantees a near-optimal trade-off between exploration and exploitation.
In this way, we provably show the \metacontroller~improves the performance as it experiences more training tasks.
%Moreover, since our \metacontroller is online, it keeps update its learned meta-representation when played on LQR instances.
Formally, we show our algorithm enjoys $\wt{O}\left(\sqrt{T}\right)$ regret (the difference between the cumulative rewards of our algorithm and the unknown optimal policy on every seen environment)  bound in the online setting.
Moreover, the algorithm is able to obtain an $\epsilon$-suboptimal policy for an unseen LQR environment after playing $\wt{O}\left(\epsilon^{-2}\right)$ environments.
%\lin{we should not use regret here. Transfer learning has no regret.} %$\widehat{O}\left(\sqrt{T}\right)$ regret.
To our knowledge, this is the first provably efficient algorithm that builds a \metacontroller~for continuous control environments.
Empirically, we simulate several physical environments to illustrate the effectiveness of our algorithm.

\paragraph{Organization}
This paper is organized as follows.
In Section~\ref{sec:rel}, we discuss related work.
In Section~\ref{sec:pre}, we formally describe the problem setup.
In Section~\ref{sec:alg}, we present our algorithm and its theoretical guarantees.
In Section~\ref{sec:exp}, we use simulation on physical environments to demonstrate the effectiveness of our approach.
We conclude in Section~\ref{sec:con}  and defer most technical proofs to the appendix.

\section{Related Work}
\label{sec:rel}
%\simon{Add discussion about the difference with Chealsea Finn's meta learning papers.}
Recently there is a large body of literature focusing on learning for control in LQR systems. The first work we are aware of is \citet{fiechter1997pac} which studies the sample complexity of LQR in the offline setting.
For the online setting, where the agent can only obtain the next state starting from the present state, the first near-optimal regret bound ($\wt{O}(\sqrt{T})$) is due to~\cite{Abbasi-Yadkori2011}, which studies the learning problem in the infinite-horizon average-case cost setting. Later on, a sequence of papers~\citep{tu2017least, dean2017sample, Dean2018a, tu2018gap, abbasi2018regret, cohen2019learning} studied this problem in similar settings, improved efficiency of the algorithms and characterized the gap between model-free and model-based approaches.

Building an agent that quickly adapts to new environment has received increasing interest in the machine learning community.
%Transfer learning for RL is also becoming increasingly interested to the learning community. 
\citet{taylor2009transfer} gave a summary for the literature before 2009. 
More recently, a sequence of theory papers \cite{lehnert2018transfer, spector2018sample, pmlr-v80-abel18a, lehnert2019reward} studied the transferability of reward knowledge, state-abstraction, and model features for Markov decision processes.
Please also refer to references therein for more details.
There are also some experimental works, e.g., \cite{santara2019extra, yu2018reusable, wu2018learning, gamrian2018transfer}, studying how to transfer knowledge from seen tasks to unseen tasks. 
Nevertheless, we are not aware of any study on how to provably perform continuous control with contexts.

%studies on the transferability even in the simple control systems like LQRs.
%\textcolor{red}{TODO: adding more related works and maybe give a more detailed overview.}
%Related work of learning to control in LQR.
%: 
%: 
%\cite{}: Temporal difference learning for LQR.
%\cite{}: more efficient and robust for LQR, regret is $T^{2/3}$.
%\cite{}: less assumptions on the system. Removing exponential dependence on the system order.
%\cite{abbasi2018regret}: model free approach for LQR regret bound.
%\cite{cohen2019learning}: efficient in solving the optimization problem but requires stronger assumptions on initialization.
%\cite{tu2018gap}: analyzes the gap of model-free and model-based for LQR.
%
%
%Transfer learning in RL:
%\cite{taylor2009transfer}: old survey paper.
%\cite{lehnert2019reward}: transfer knowledge of rewards.
%\cite{pmlr-v80-abel18a}: state-abstraction for life long learning.
%\cite{spector2018sample}: transfer from architecture priors.
%\cite{santara2019extra}: transfer guided policy search.
%\cite{lehnert2018transfer}: transfer with model features.
%\cite{yu2018reusable}: looks similar to ours. Double check.
%
%
%Empirical papers:
%\cite{wu2018learning}.
%\cite{gamrian2018transfer}

\section{Preliminaries}
\label{sec:pre}
\paragraph{Notations.}
We begin by introducing necessary notations.
We write $[h]$ to denote the set $\left\{1,\ldots,h\right\}$.
We use $I_d\in \RR^{d\times d}$ to denote the $d \times d$ identity matrix.
We use $0_{d\times d'}$ to represent the all-zero matrix in $\RR^{d\times d'}$. 
If it is clear from the context, we omit the subscript $d\times d'$. 
%For any finite set $S$, we write $\unif\left(S\right)$ to denote the uniform distribution over $S$ and $\simplex\left(S\right)$ to denote the probability simplex.
Let $\norm{\cdot}_2$ denote the Euclidean norm of a vector in $\mathbb{R}^d$. 
For a symmetric matrix $\mat{A}$, let $\normop{\mat{A}}$ denote its operator norm and $\lambda_{i}\left(\mat{A}\right)$ denote its $i$-th eigenvalue.
Throughout the paper, all sets are multisets, i.e., a single element can appear multiple times. 

\paragraph{Finite Horizon Linear Quadratic Regulator.}
We now formally define the finite horizon Linear Quadratic Regulator (LQR) problem.
In the LQR problem, there is a state space $\mathcal{X}\subset \mathbb{R}^d$ and a closed action space $\mathcal{U}\subset\mathbb{R}^{d'}$.    Suppose we always start from the initial state $x_1=x_{\mathrm{init}}\in \cX$ and play for $H$ steps. Then at a state $x_h\in \mathcal{X}$, if an action $u_h\in\mathcal{U}$ is played, the next state is given by
\begin{align}\label{equ:lqr}
x_{h+1}  = A x_h + B u_h + w_{h+1},
\end{align}
where $A, B$ are matrices of proper dimension and $w_{h+1}$ is a zero-mean random vector.
Here $A,B$ can be viewed as the succinct representation of this LQR since as will be explained below, given $A,B$, one can easily infer the optimal policy for this LQR. 
 For simplicity, we denote 
\[
M = [A, B], \text{ and } y_h=[x_h^\top, u_h^\top]^\top\in \mathbb{R}^{d+d'}.
\]
Now the state transition can be rewritten as
$
x_{h+1} = My_h + w_{h+1}.
$
For the ease of presentation, we assume that the covariance matrix of noise vector $w_{h+1}$ is $\mathbb{E}(w_{h+1}w_{h+1}^\top) = I_d$. 
Our analysis follows similarly if the covariance matrix is not $I_d$ (see e.g. Remark 3 in \cite{Abbasi-Yadkori2011}). 
After each step, the player receives an immediate cost
$
x_{h}^\top Q_h x_h + u_h^\top R_h u_h,
$
where $Q_h, R_h$ are positive definite matrices of proper dimensions.  %The game ends in $H>0$ steps. 
At a terminal state $x_H$, there is no action to be played, and the player receives a terminal cost
$
x_H^\top Q_H x_H,
$
where $Q_H$ is a positive semi-definite matrix of proper dimension. 
The goal of the player is to find a policy $\pi: (\cX\times\cU)^*\times \cX\rightarrow \cU$, which is a function that maps the trajectory $\{(x_i, u_i)\}_{i=1}^{h-1}\cup \{x_{h}\}$ to the next action $u_{h}$, such that the following objectives are minimized:
\[
\bigg\{J^{\pi}_h(M, x):= \mathbb{E}\bigg[\bigg(\sum_{h'=h}^{H-1} x_{h}^\top Q_h x_h + u_h^\top R_h u_h\bigg) + x_H^\top Q_fx_H~\bigg|~ x_h = x\bigg]
\bigg\}_{h\in[H]},
\]
where the action $u_h$ is given by $u_h=\pi[(x_1, u_1), (x_2, u_2), \ldots, (x_{h-1}, u_{h-1}), x_{h}]$, and the expectation is over the randomness of $w_h$ and $\pi$.  %If it is clear from the context, we ignore $M$ in the $J^{\pi}$.

It is well-known that the optimal policy $\pi^*$ is Markovian \cite{puterman2014markov}, i.e., it only depends on the present state. For an unconstrained action space $\cU$, we have
\[
\forall x\in \cX, h\in [H-1]:\quad \pi_h^*(M, x):= K_h(M) x
\]
where $M = [A, B]$ and $K_h(M)$ is a matrix that will be defined shortly. It is also known (see e.g. \cite{bertsekas1996dynamic}) that the optimal cost function $J^*_h(x):=J^{\pi^*}_h(x)$ is given by
\begin{align}
J^*_h(M, x):=x^\top P_h(M) x + C_h(M)
=\inf_{\pi}J_h^{\pi}(M, x)\label{eq:optimal-value}
\end{align}
where 
\begin{align}
\label{eqn:ph}
P_h(M)= 
\begin{cases}
Q_h+ A^\top  P_{h+1}(M) A - A^\top P_{h+1} B(R_h+B^\top P_{h+1}(M) B)^{-1}B^\top P_{h+1}(M) A & h < H \\
Q_H & h = H
\end{cases}
\end{align}
%\lin{I don't find a reference for this. Maybe I should  put a proof in appendix for completeness.}
and 
\[
C_h(M) = \begin{cases}
C_{h+1}(M) + \mathbb{E}_{w_{h+1}}\big[w_{h+1}^\top P_{h+1}(M) w_{h+1}] & h < H\\
0 & h = H
\end{cases}.
\]
We now define $K_h(M)$ as
\begin{align}
K_h(M):=-(R_h+B^\top P_{h+1}(M) B)^{-1}B^\top P_{h+1}(M) A.\label{eq:km}
\end{align}
Note that the optimal value Equation ($\ref{eq:optimal-value}$) satisfies Bellman equations,
\[
\forall h\in [H-1]:\quad J_h^*(M, x)=x^\top Q_h x + {\pi^*(x)}^\top R_h \pi^*(x) + \mathbb{E}\big[J_{h+1}^*(Ax+B\pi^*(x) + w)\big]
\]
and 
\[
\forall h\in [H-1]:\quad J_h^*(M, x)=x^\top Q_h x + \min_{u}\mathbb{E}[u^\top R_hu + J_{h+1}^*(Ax+Bu + w)].
\]
%To ensure controllability, we make the following assumption.
%\begin{asmp}
%	\label{asmp:reg-pm}
%	For all $h\in [H]$, $\|P_h(M)\|_2\le c_q$ for some parameter $c_q>0$.
%\end{asmp}

Now we have shown that if we are given $A$ and $B$, then we can obtain the optimal policy directly.
In this paper, we deal with setting where $A$ and $B$ are \emph{unknown} and we need to use \metacontroller~to decode $A$ and $B$ from the contexts of the current LQR, as specified below.

\paragraph{Learning to Control LQR with Contexts}
In the continuous control with contexts setting, in each episode we observe a context
\[(C, D)\sim \mu,\] where $\mu$ is a distribution on $\RR^{p\times d}\times\RR^{p'\times d'} $.
The context $[C,D]$ encodes the information of the environment.
Formally, the representation ($[A,B]$) of this environment can be decoded from the context via a decoding matrix $\Theta_* \in \mathbb{R}^{d \times (p+p')}$:
%Then the agent is facing an LQR problem with the following parameters:
\begin{align}
	[A, B] = \Theta_*\cdot \bigg[\begin{array}{cc}C& 0_{p\times d'}\\0_{p'\times d} & D\end{array}\bigg].
\end{align}
From now on, to emphasize that the representation of LQR can be decoded from $\Theta_*$, we write
\begin{align}
\label{eqn:mtcd}
M_{\Theta_*, C, D}:=
%[A_*(C), B_*(D)] := 
\Theta_*\cdot \bigg[\begin{array}{cc}C& 0_{p\times d'}\\0_{p'\times d} & D\end{array}\bigg] = [A,B].
\end{align}
If it is clear from the context, we ignore $[C,D]$ for notational simplicity.
Note the optimal decoder $\Theta_*$ is unknown to the agent and the goal is to learn $\Theta_*$ from contexts and interactions with the environment.
Below we formally define the problem that we study.
\begin{defn}[Contextual Transfer Learning Problem]
Build an agent that plays on $K$ LQR games (one trajectory per game) with context pairs $\{(C^{(1)}, D^{(1)}), (C^{(2)}, D^{(2)}), \ldots, (C^{(K)}, D^{(K)})\}\sim \mu$, for some integer $K\ge 0$ such that for another new context pair $(C, D)\sim \mu$, the agent outputs a policy $\pi$ based on $(C,D)$ which satisfies
\[
\EE[J^{\pi}_h(M_{\Theta_*, C, D}, ~x_1) - J^{*}_h(M_{\Theta_*, C, D}, ~x_1)] \le \epsilon
\]
for some given target accuracy $\epsilon > 0$. 
\end{defn}
Here $K$ is the sample complexity which ideally scales \emph{polynomially} with $1 / \epsilon$ and problem-dependent parameters.
The performance of the agent can also be measured by regret, as defined below.
\begin{align}
\label{eq:reg}\mathrm{Regret}(KH):= \sum_{k=1}^{K}J_1^{\wt{\pi}^{(k)}}\Big(M_{\Theta_*, C^{(k)}, D^{(k)}}, ~x_1\Big) - J_1^*\Big(M_{\Theta_*, C^{(k)}, D^{(k)}}, x_1\Big),
\end{align}
where $\wt{\pi}^{(k)}$ is the policy played at episode $k$ by the agent.
This quantity measurse the sub-optimality of policies the agent played in the first $K$ episodes.

\begin{rem}
We consider matrix-type linear maps from context to the representation only for sake of presentation.
Our algorithm and analysis can be readily extended to other linear maps, e.g.,
$
[A_*(C), B_*(D)] : = f(C,D)
$ 
for some unknown linear function $f$.
%\lin{Adding some notes for other models of contexts.}
\end{rem}

\section{Main Algorithm}
\label{sec:alg}
%\section{LQR Transfer Learning}
%\label{sec:alg}
In this section, we first describe the algorithm and then present its sample complexity guarantees.
\paragraph{Algorithm}
%We train an agent via playing \emph{a single} episode on each context pair we obtain.
%After seeing $K$  episodes of different contexts sampled from distribution $\mu$, we output a function, that maps a context to a policy.
We describe the high-level idea of the algorithm below.
The agent maintains a \metacontroller~that maps the context $(C,D)$ to the representation $(A,B)$.
We denote $\Theta^{(k)}$ the \metacontroller~at the $k$-th episode.
Initially, we know nothing about $\Theta_*$, so we initialize our \metacontroller~by setting $\Theta^{(1)} = 0\in \mathbb{R}^{d\times p}$. 
%Suppose we train the agent for a total of $K$ episodes. Suppose at episode $k$, we obtain context $C^{(k)}, D^{(k)}$.
%We then compute a policy $\pi^{(k)}$ using $\Theta^{(k)}$.
% We then play policy $\pi^{(k)}$. 
At the $k$-th episode,  the agent plays policy $\pi^{(k)}$ and in each time step $h\in [H-1]$, it collects data 
\[
x^{(k)}_h, u^{(k)}_h, x^{(k)}_{h+1}, \quad z^{(k)}_{h}\gets\bigg[ \begin{array}{c}
\vspace{1mm}
C^{(k)}{x^{(k)}_h}\\
D^{(k)}{u^{(k)}_h}
\end{array}\bigg],
\]
where $z^{(k)}_h$ can be viewed as the \emph{context regularized} observation.
We now describe how to obtain policy $\pi^{(k)}$. We first solve the following optimization problem
\[
\widetilde{\Theta}^{(k)} = \arg\min_{\Theta\in \cC^{(k)}} J_1^*
\Big(M_{\Theta, C^{(k)}, D^{(k)}}, x_1^{(k)}\Big),
%\Bigg(\Theta\cdot \bigg[\begin{array}{cc}C^{(k)}& 0\\0 & D^{(k)}\end{array}\bigg], \quad x_1^{(k)}\Bigg)%\label{eq:solve-fou}
\]
where $J_1^*$ is given by Equation (\ref{eq:optimal-value}), 
and the confidence set $\def\cC{\mathcal{C}}\cC^{(k)}$ will be defined shortly.
$\cC^{(k)}$ represents our confidence region on $\Theta_*$.
Since we choose the one that minimizes the cost, this represents the principle ``optimism in the face of uncertainty" and it is the key to balance exploration and exploitation which will be clear in the proof.
 Notice that the above optimization problem is a polynomial optimization problem. Then the policy is given by  
\[
\pi^{(k)}_h(x):= K_h\Big(M^{(k)}\Big)\cdot x
\text{ where }
M^{(k)} =M_{\Theta^{(k)}, C^{(k)}, D^{(k)}}:= \Theta^{(k)}\cdot \bigg[\begin{array}{cc}C^{(k)}& 0\\0 & D^{(k)}\end{array}\bigg],
\]
and $K_h$ is given by Equation~(\ref{eq:km}).
After episode $k\in [K]$, we use the following ridge regression to the update \metacontroller
\[
\Theta^{(k)}  = 	\left(\Big(V^{(k+1)}\Big)^{-1}W^{(k+1)}\right)^\top,
\]
where 
\[
V^{(k+1)} = I +\sum_{k'=1}^k\sum_{h=1}^{H-1}z^{(k')}_h z^{(k')\top}_h\quad\text{and}\quad
W^{(k+1)}
= \sum_{k'=1}^k\sum_{h=1}^{H-1}z^{(k')}_h x^{(k')\top}_{h+1}.
\]
After playing $K$ episodes, the algorithm outputs a 
$\wt{\Theta}$ by picking one from $\{\wt{\Theta}^{(k)}\}_{k\in[K]}$ uniformly at random.
Now for a new task with its context, our learned policy map is given by:
\begin{align}
\label{eq:rand-pi}
\forall C,D\sim \mu, x\in \cX, h\in[H-1]:\quad\wt{\pi}_{C,D,h}(x)
= K_h\bigg(\wt{\Theta}\cdot \bigg[\begin{array}{cc}C& 0\\0 & D\end{array}\bigg]\bigg)\cdot x.
\end{align}
The formal algorithm is presented in Algorithm~\ref{alg:core-rl}.
%where $K_h$ is defined in Equation~\ref{eq:km}.
%The analysis of the algorithm is presented in the next section.

\begin{algorithm*}[t]
	\caption{Linear Continuous Control with Contexts\label{alg:core-rl}}\small
	\label{alg:main}
	\begin{algorithmic}[1]
		\State 
		\textbf{Input} Total number of episodes $K$;
		\State \textbf{Initialize} $\Theta^{(1)}\gets 0\in\RR^{d\times 2p}$, $V^{(1)}\gets I_{2p,2p}$, 
		$W^{(1)}\gets 0\in\RR^{2p\times d}$;
		\For{episode $k=1, 2, \ldots, K$}
		\State Let $x^{(k)}_1\gets x_{\rm init}$,  $V^{(k+1)}\gets V^{(k)}$, $W^{(k+1)}\gets W^{(k)}$;
		\State Obtain context $[C^{(k)}, D^{(k)}]\sim\mu$;
		\State Solve
		\begin{align}
			\widetilde{\Theta}^{(k)} = \arg\min_{\Theta\in \cC^{(k)}} J_1^*\Big(M_{\Theta, C^{(k)}, D^{(k)}},~ x_1^{(k)}\Big)\label{eq:solve-fou}
		\end{align}
		\State where $J_1^*$ is given by Equation~\ref{eq:optimal-value}, and $\cC^{(k)}$ is defined in Equation~\ref{eq:confidence-ball};
		\For{stage $h=1, 2, \ldots, H-1$}
		\State Let the current state be $x^{(k)}_h$;
		\State  Play action $u^{(k)}_{h} \gets K_h\big(M_{\Theta^{(k)}, C^{(k)}, D^{(k)}}\big)\cdot x^{(k)}_h$, where $K_h$ is defined in Equation~\ref{eq:km};
		\State Obtain the next state $x^{(k)}_{h+1}$;
		\State Let $z^{(k)}_h\gets \bigg[ \begin{array}{c}
		\vspace{1mm}
		C^{(k)}{x^{(k)}_h}\\
		D^{(k)}{u^{(k)}_h}
		\end{array}\bigg]$;
		\State Update: $V^{(k+1)}\gets V^{(k+1)} + 
		z^{(k)}_hz^{(k)\top}_h$;
		%\Big[C^{(k)}x^{(k)}_h; D^{(k)}u^{(k)}_h\Big]\Big[C^{(k)}x^{(k)}_h; D^{(k)}u^{(k)}_h\Big]^\top$;
		\State Update: $W^{(k+1)}\gets W^{(k+1)}+ %\Big[C^{(k)}x^{(k)}_h; D^{(k)}u^{(k)}_h\Big]
		z^{(k)}_h\Big(x^{(k)}_{h+1}\Big)^\top$;
		\EndFor
		\State \label{line:linreg} Compute 
		$
			\Theta^{(k+1)\top}
			\gets 
			\Big(V^{(k+1)}\Big)^{-1}W^{(k+1)};
		$
		\EndFor
	\State\textbf{output} $\wt{\Theta}^{(k)}$ 
%	(or ${\Theta}^{(k)}$) 
	where $k$ is chosen from $[K]$ uniformly at random.
	\end{algorithmic}
\end{algorithm*}
\subsection{Algorithm Analysis}
To present the analysis of the algorithm, we first introduce our assumptions.
\begin{asmp}
	\label{asmp:regularity}
The contexts and LQR satisfy the following properties.
\begin{itemize}
\item $\forall h\in [H]$, $\|P_h(M)\|_2\le c_q$ for some parameter $c_q>0$.
\item %For all %$(C,D)\in \supp(\mu)$, $\|[\Theta_* C, \Theta_* D]\|_F\le c_M$,
 $\|\Theta_*\|_F\le c_{\Theta}$;
\item $\forall h\in[2,H],i\in[d]:\quad \|w_{h}\|_2< \infty$ and $\forall \gamma >0, ~\mathbb{E}[\gamma w_{h,i}]\le \exp(\gamma^2c_w^2/2)$;
\item $\forall x\in \cX, u\in \cU, (C,D)\in \supp(\mu): \quad \|Cx\|_2+\|Du\|_2^2\le c_x^2$, $\|x\|^2+\|u\|_2^2\le c_x^2$;
\item $\forall (C,D)\in \supp(\mu), x\in \cX, h\in [H]$: $K_h(M_{\Theta_*, C, D})\cdot x\in \cU$.
\end{itemize}
where $c_\Theta, c_w, c_x$ are some positive parameters.
\end{asmp}
The first assumption is standard to ensure controllability.
The second is a regularity condition on the optimal \metacontroller~$\Theta_*$. 
The third assumption implies the noise $w$ is sub-Gaussian and imposes boundedness of the noise $w$. 
%This is slightly stronger than just assuming sub-Gaussianity. If we change the assumption to just sub-Gaussianity, we only need to modify our analysis by truncating the sub-Gaussian random variable by a high probability event. Thus for the sake of representation, we take the almost sure boundedness for $w$. 
The fourth assumption is a regularity condition on the observation.
%The fouth assumption comes directly from the first and the third assumption and the fact that $H$ is finite. 
The last assumption guarantees the optimal controller for the unconstrained LQR problem is realizable in our control set $\cU$.
Given these assumptions, We are now ready to define confidence set $\cC^{(k)}$ as follows.

\begin{align}
\cC^{(k)} = \Big\{\Theta:~\mathrm{tr}&\big[\big(\Theta-\Theta^{(k)}\big) V^{(k)}\big(\Theta-\Theta^{(k)}\big)^\top\big]\le \beta^{(k)}, \nonumber\\
&~\text{and}~\forall h\in [H], (C, D)\in\supp(\mu), ~\big\|P_{h}\big(M_{\Theta, C,D}\big)\big\|_2\le c_q%, \|\Theta[C,D]\|_2\le c_{\Theta}
\Big\}
\label{eq:confidence-ball},
\end{align}
where $P_h$ is given by Equation~(\ref{eqn:ph}) and $\beta^{(k)}$ %will be defined shortly. we set an appropriate bound for $\beta^{(k)}$, which is
is defined as follows,
\begin{align}\label{equ:beta_k}
\beta^{(k)}=\Big(c_{\Theta}
+c_w\sqrt{2d\big(\log d + p\log(1+kHc^2_x/p)/2 + \log\delta^{-1}\big)}\Big)^2.
\end{align}

With the above assumptions,
the guarantee of Algorithm~\ref{alg:core-rl} is formally presented in the next theorem.
\begin{thm}
	\label{thm:main}
Suppose we run Algorithm~\ref{alg:core-rl} 
for 
\[
K\ge\frac{c_{H,c_q,c_x,c_{\Theta},c_w}'\cdot dp^2\cdot\log^3(dK\delta^{-1})}{\epsilon^2}
\]
 episodes, for some parameter $c_{H,c_q,c_x,c_{\Theta},c_w}'$ that depends polynomially on  $H$, $c_q, c_x, c_\Theta, c_w$, with probability at least $1-\delta$, 
%For a context $[C, D]\sim\mu$ and $M_*=[\Theta_* C, \Theta_* D]$, 
for  $\wt{\pi}_{C,D}$ defined in Equation~\ref{eq:rand-pi},
we have
%We have
\begin{align}
\label{reg:pac}
\underset{[C,D]\sim\mu}{\EE} \Big[\underset{\wt{\pi}_{C,D}}{\EE}\Big(J_1^{\wt{\pi}_{C,D}}([\Theta_* C, \Theta_* D], x_1)\Big) - J_1^{*}([\Theta_* C, \Theta_* D], x_1)\Big]
\le \epsilon.
\end{align}
\end{thm}
Theorem~\ref{reg:pac} states that after playing polynomial number of episodes, our agent can learn a \metacontroller~$\widetilde{\Theta}$ such that given a new LQR with contexts $(C,D)$, this decoder can turns the contexts into a near-optimal policy $\wt{\pi}_{C,D}$ \emph{without any training} on the new LQR.
Note this is the desired agent we want to build as described in the introduction.
We emphasize again that this is the first provably efficient algorithm that builds a \metacontroller~for continuous control environments.

\begin{rem}
Via similar analysis, it is easy to show that if the output $\wt{\Theta}$ is picked uniformly at random from $\{{\Theta}^{(k)}\}_{k\in[K]}$, the policy achieves similar guarantees.
%	For simplicity, we only show the analysis for 
%	picking $\wt{\Theta}$ from $\{\wt{\Theta}^{(k)}\}_{k\in[K]}$.
\end{rem}

%The proof is deferred to Section~\ref{sec:proofmain}.
In fact, Theorem~\ref{thm:main} is implies by the following regret bound of our algorithm.
%
%We also note that our algorithm achieves a small ``regret''.
%The ``regret'' compares the cost of the algorithm to the cost incurred by the optimal policy for every context, which is formally defined as,
%\begin{align}
%\label{eq:reg}\mathrm{Regret}(KH):= \sum_{k=1}^{K}J_1^{\wt{\pi}^{(k)}}\Big(M_{\Theta_*, C^{(k)}, D^{(k)}}, ~x_1\Big) - J_1^*\Big(M_{\Theta_*, C^{(k)}, D^{(k)}}, x_1\Big),
%\end{align}
%where $\wt{\pi}^{(k)}$ is the policy played at episode $k$ by the algorithm.
%We then show the following proposition.
\begin{prop}
	\label{prop:main}
	With probability at least $1-\delta$, 
	$$
	\mathrm{Regret}(KH)\le  c_{H}'\cdot d^{1/2}p\cdot \log^{3/2}(dKHc_x\delta^{-1})\cdot\sqrt{KH}.
	$$
	where $c_{H}'$ is a constant depending only polynomially on $H$, $c_q, c_x, c_M, c_w$.
\end{prop}
By the definition of regret, this proposition implies that the performance of the agent improves as it sees more environments.
%In the following we layout a proof sketch of Proposition~\ref{prop:main}.

%With  this  proposition, the proof for Theorem~\ref{thm:main} is relatively standard.
%We thus postpone its proof to Section~\ref{sec:proofmain}.

\section{Experiments}
\label{sec:exp}
In this section, we validate the effectiveness of our algorithm via numerical simulations.

We perform experiments on a path-following task.
In this task, we are given a trajectory $z_1^*, z_2^*, \ldots, z_H^* \in \mathbb{R}^2$.
Our goal is to exert forces $u_1, u_2, \ldots, u_m \in \mathbb{R}^2$ on objects with different (measurable) masses to minimize the total squared distance plus the sum of the squared Euclidean norms of the forces, i.e., $\sum_{h = 1}^H \|z_h - z_h^*\|^2 + \|u_i\|_2^2$.
Each state $x_h = [z_h; v_h] \in \mathbb{R}^4$ is a vector whose first two dimensions represent the current position and the last two dimension represent the current velocity.
In each stage $h$, we may exert a force $u_h \in \mathbb{R}^2$ on the object, which produces an accelerations $\frac{u_h}{m} \in \mathbb{R}^2$.
The dynamics of the system can be described as 
\begin{align}\label{equ:law}
\begin{cases}
z_{h + 1} = z_{h} + v_h\\
v_{h + 1} = k \cdot v_h + u_h / m
\end{cases}
\end{align}
where $0 < k \le 1$ is the decay rate of velocity induced by resistance.
In our setting, the decay rate of velocity $k$ is fixed (encoded in $\Theta_*$), where the mass of the object $m$ is drawn from the uniform distribution over $[0.1, 10]$.
In our experiments, we set the noise vector $w_h$ in the dynamics of the LQR system (cf. Equation~\ref{equ:lqr}) to be a Gaussian random vector with zero mean and covariance $10^{-4} \cdot I$.
In each episode, we receive an object with mass $m$ where $m$ is draw from the uniform distribution over $[0.1, 10]$, train one trajectory using that object, and the goal is to recover the physical law described in Equation~\ref{equ:law} so that our model can deal with objects with unseen mass $m$.
Please see Appendix~\ref{sec:setting} for the concrete value of $\Theta_*$, $Q$ and $R$ and the distribution of $C$ and $D$.

In our experiments, we use $100$ different masses as {\em training masses} (fixed among all experiments), and use $100$ different masses as {\em test masses} (again fixed among all experiments). All the training masses and test masses are drawn from the uniform distribution over $[0.1, 10]$.
We implement a practical version of Algorithm~\ref{alg:main}.
In particular, instead of solving the optimization problem in Equation~\ref{eq:solve-fou} exactly, we sample $100$ different $\Theta$ from $\cC^{(k)}$ uniformly at random, and choose the $\Theta$ which minimizes the objective function. 
Moreover, instead of using the theoretical bound for $\beta^{(k)}$ in Equation~\ref{equ:beta_k}, we treat $\beta^{(k)}$ as a tunable parameter and set $\beta^{(k)} = 10^4$ in our experiments to encourage exploration at early stage of the algorithm. 
We use two different metrics to measure the accuracy of the learned model. 
First, we use $\|\Theta_k - \Theta_*\|_F$ where $\Theta_k$ is calculated in Line~\ref{line:linreg} to measure the accuracy of the learned $\Theta$.
Moreover, using the learned $\Theta$, we test on $100$ objects whose masses are the $100$ test masses to calculate the control cost $\sum_{h = 1}^H \|z_h - z_h^*\|^2 + \|u_i\|_2^2$.
We compare the control cost of the learned $\Theta$ and the optimal control cost, and use the mean value of the differences (named mean control error) to measure the accuracy.

In all experiments we fix $H = 20$.
We use three different types of trajectories: unit circle, parabola $y = x^2$ with $x \in [0, 1]$ and Lemniscate of Bernoulli with $a = 1$\footnote{\url{https://en.wikipedia.org/wiki/Lemniscate_of_Bernoulli}.}.
For all three types of trajectories we use their parametric equation $x = x(t)$ and $y = y(t)$, divide the interval $[0, 1]$ evenly into $H$ parts, and set $t$ to be the endpoints of these parts.
We use these $t$ values to define the trajectory  $z_1^*, z_2^*, \ldots, z_H^* \in \mathbb{R}^2$.
We set the decay ratio $k$ to be $k = 1$ or $k = 0.7$ in our experiments.

We plot the accuracy of the learned model in Figure~\ref{fig:acc}.
\begin{figure*}[t]
\centering
\includegraphics[scale=0.4]{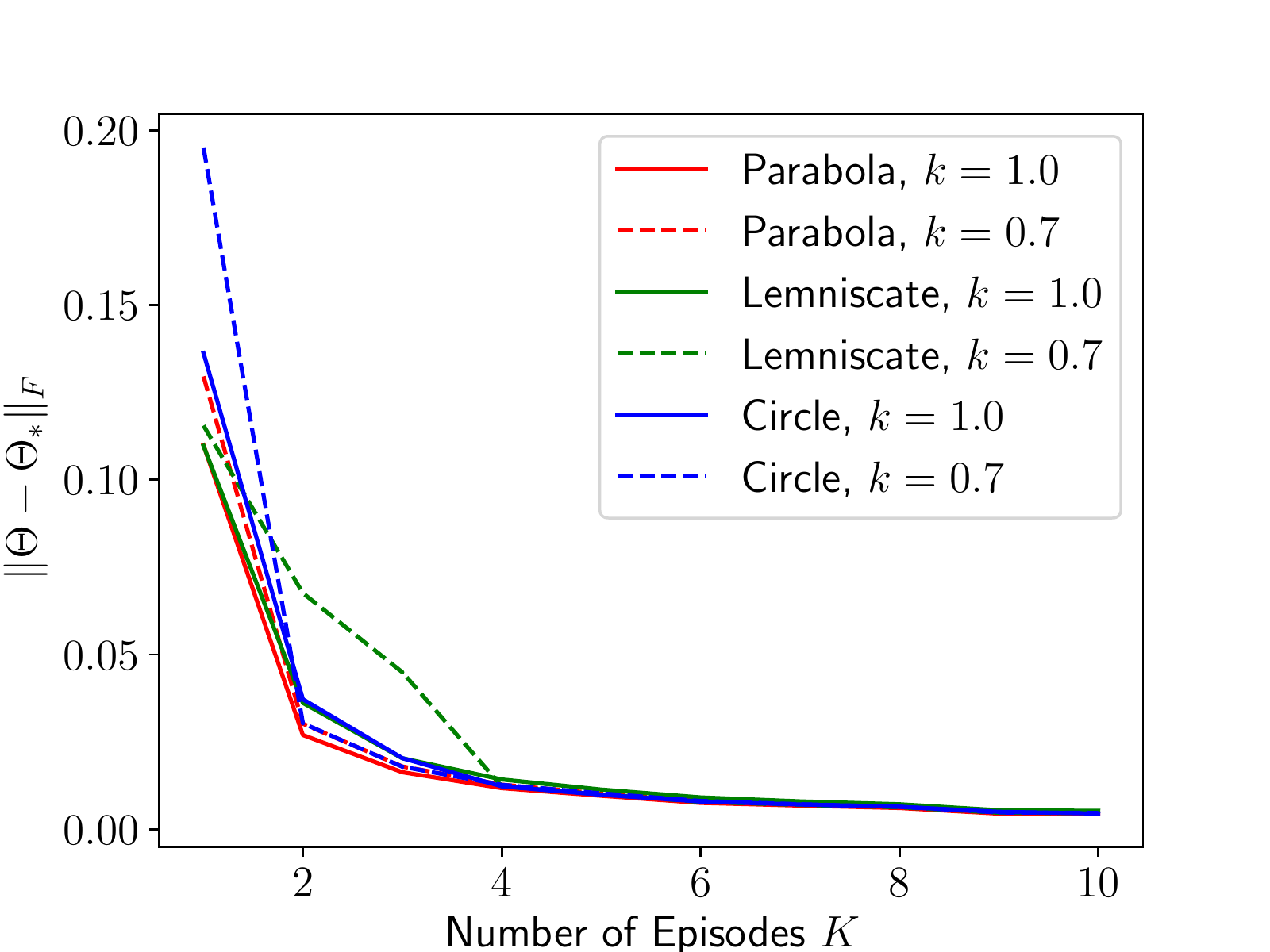}
\includegraphics[scale=0.4]{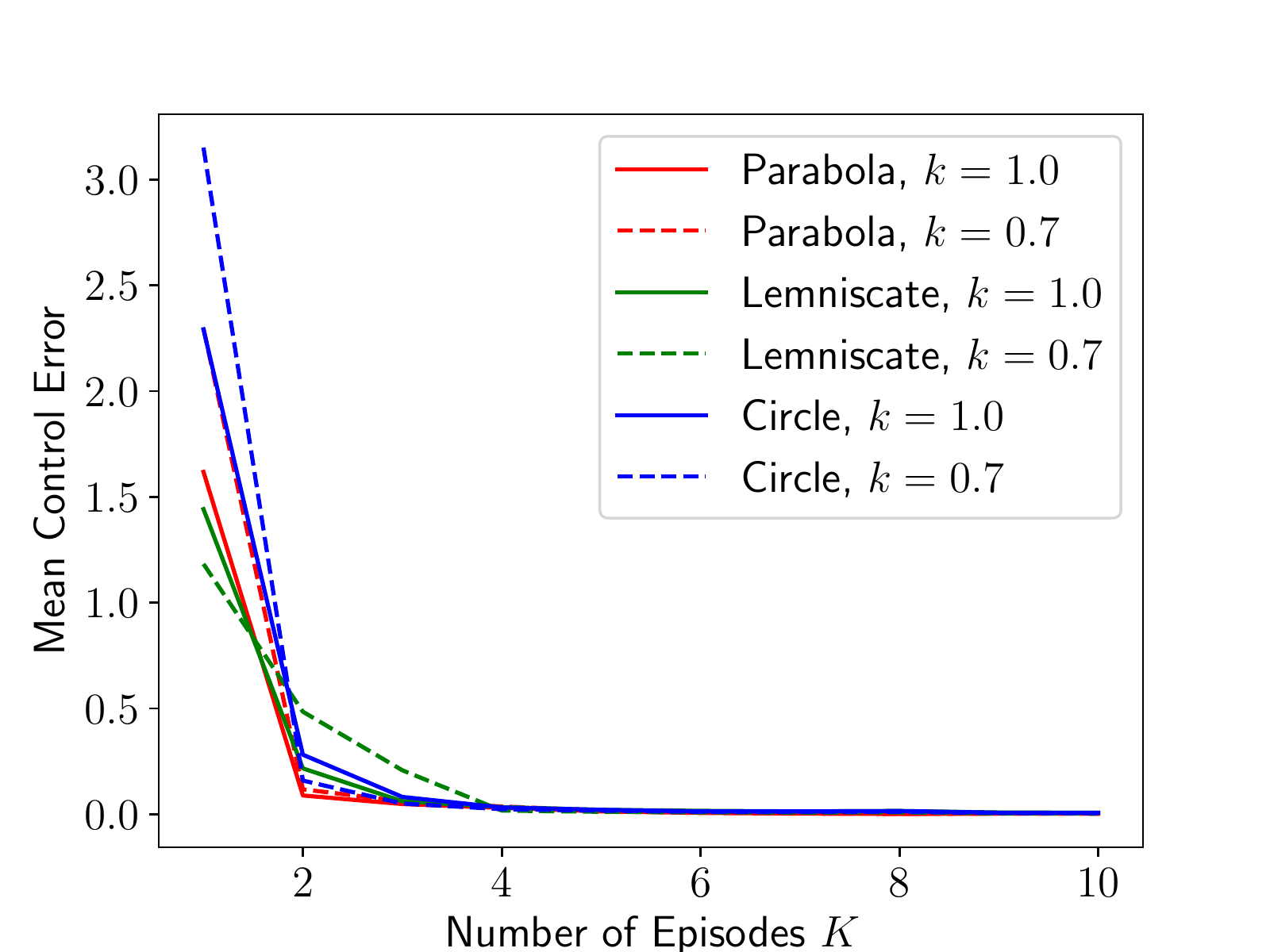}
\caption{$\|\Theta - \Theta_*\|_F$ and Mean Control Error.}
\label{fig:acc}
\end{figure*}
Here we vary the number of training episodes (the number of training masses) and observe its effect on the accuracy. 
It can be observed that our algorithm achieves an satisfactory accuracy using only $5$ episodes.
%Our experiments also verify the $\widetilde{O}(1 / \varepsilon^2)$ sample complexity of our algorithm.
We also illustrate trajectories obtained by our resulting controllers in Figure~\ref{fig:exp}.
From Figure~\ref{fig:exp}, it is clear that as the agent plays more environments, it achieves better performance.
\begin{figure*}
	\includegraphics[width=0.48\textwidth]{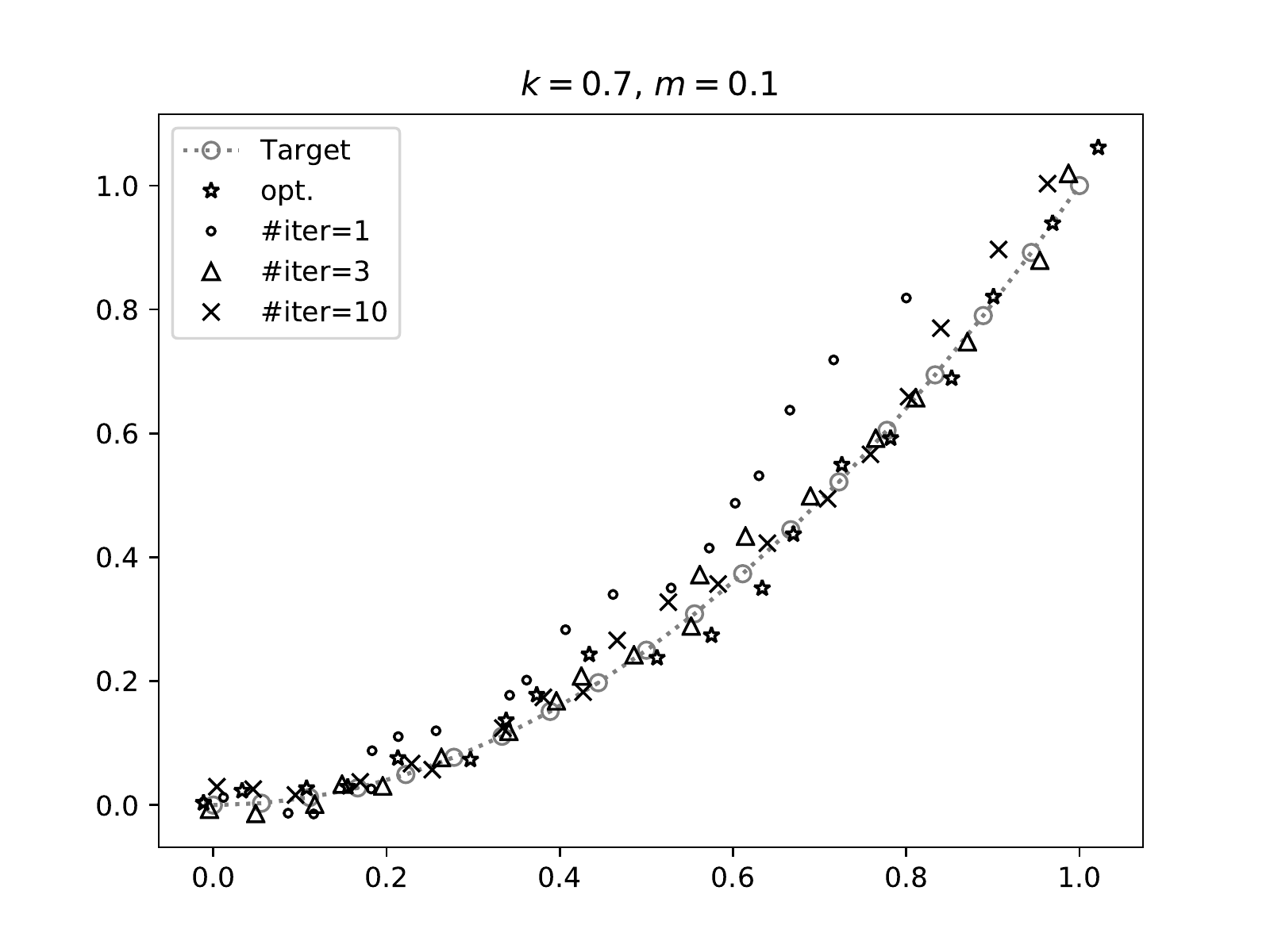}
	\includegraphics[width=0.48\textwidth]{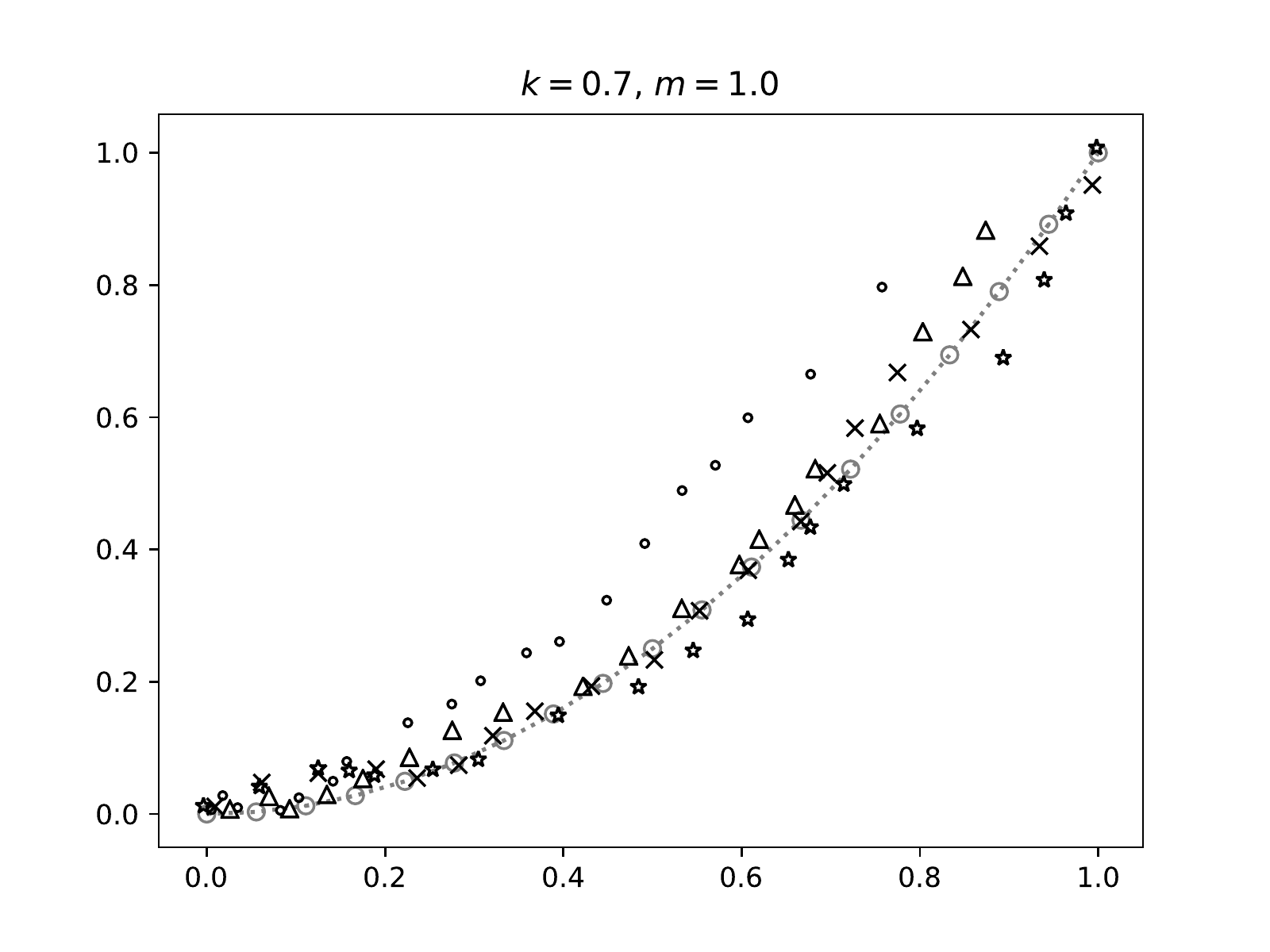}\\
	\includegraphics[width=0.48\textwidth]{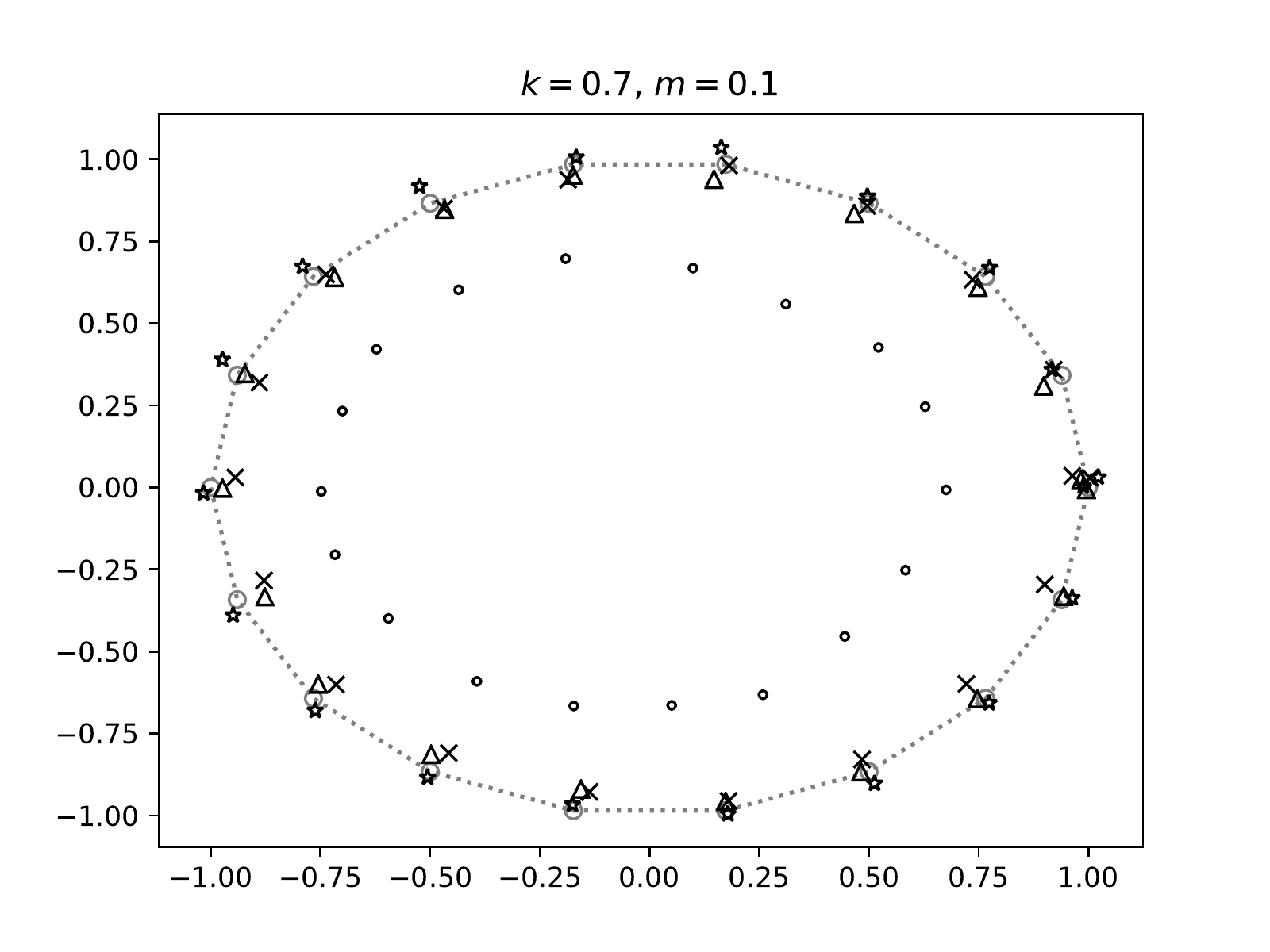}
	\includegraphics[width=0.48\textwidth]{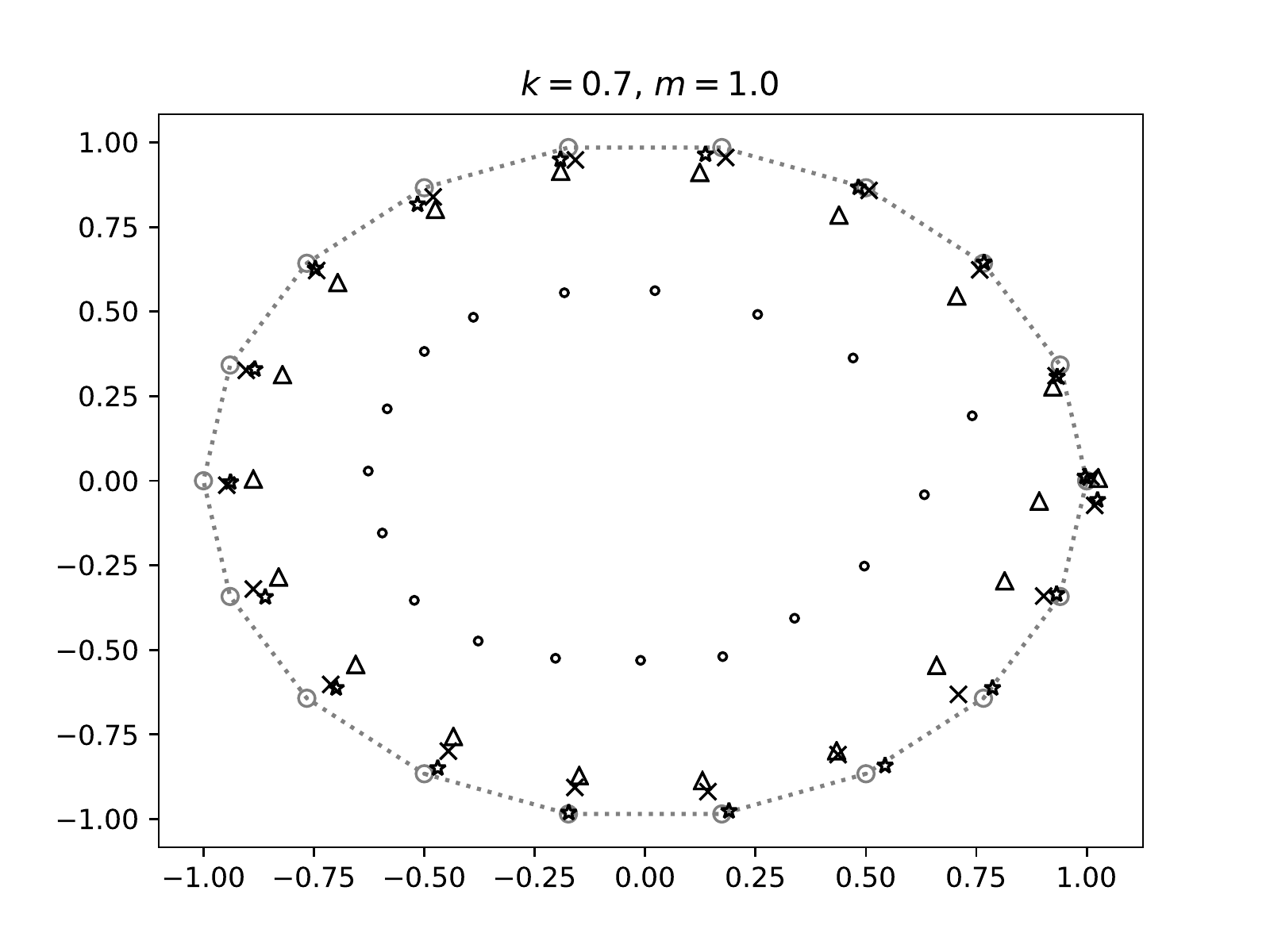}\\
	\includegraphics[width=0.48\textwidth]{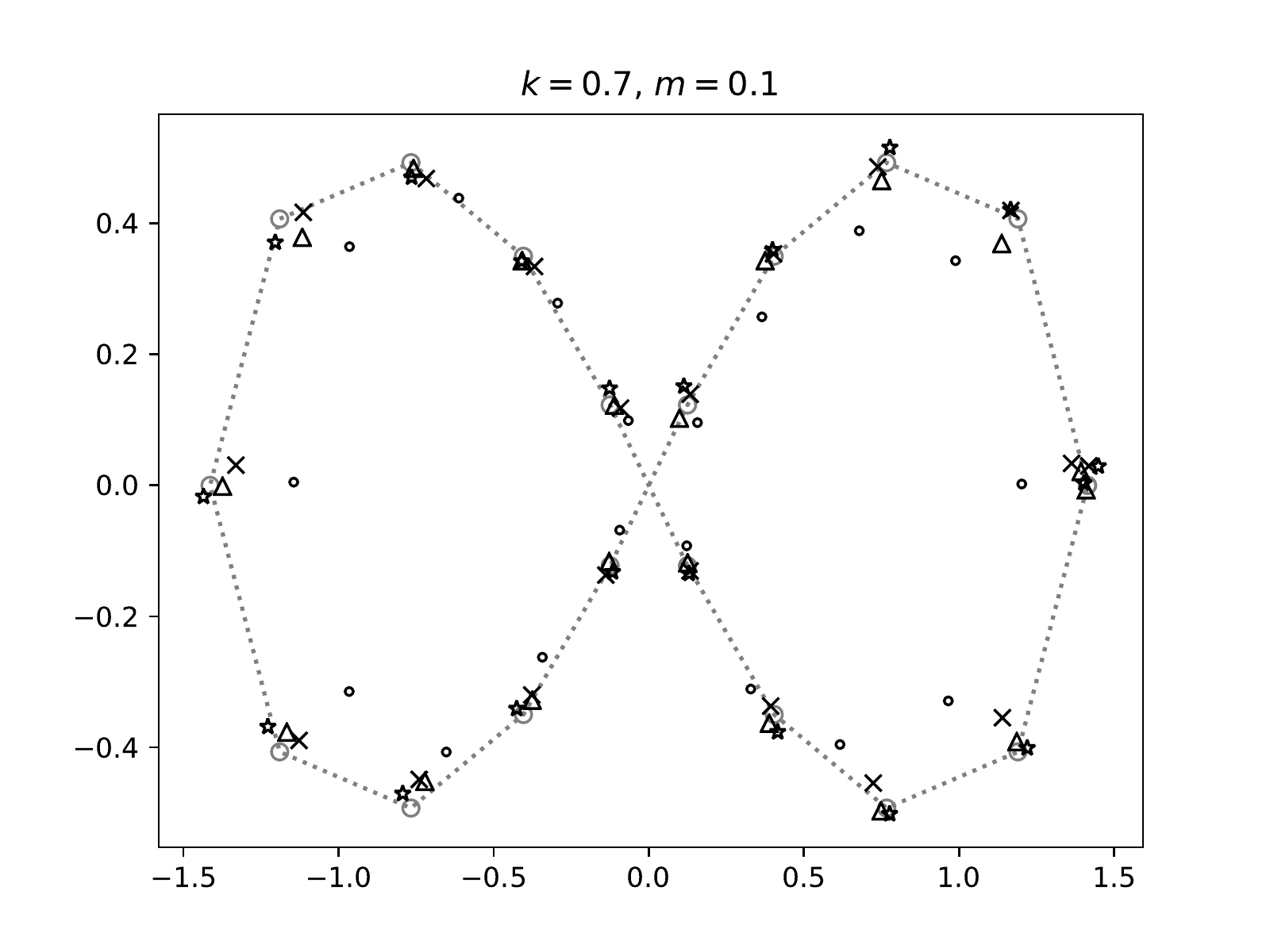}
	\includegraphics[width=0.48\textwidth]{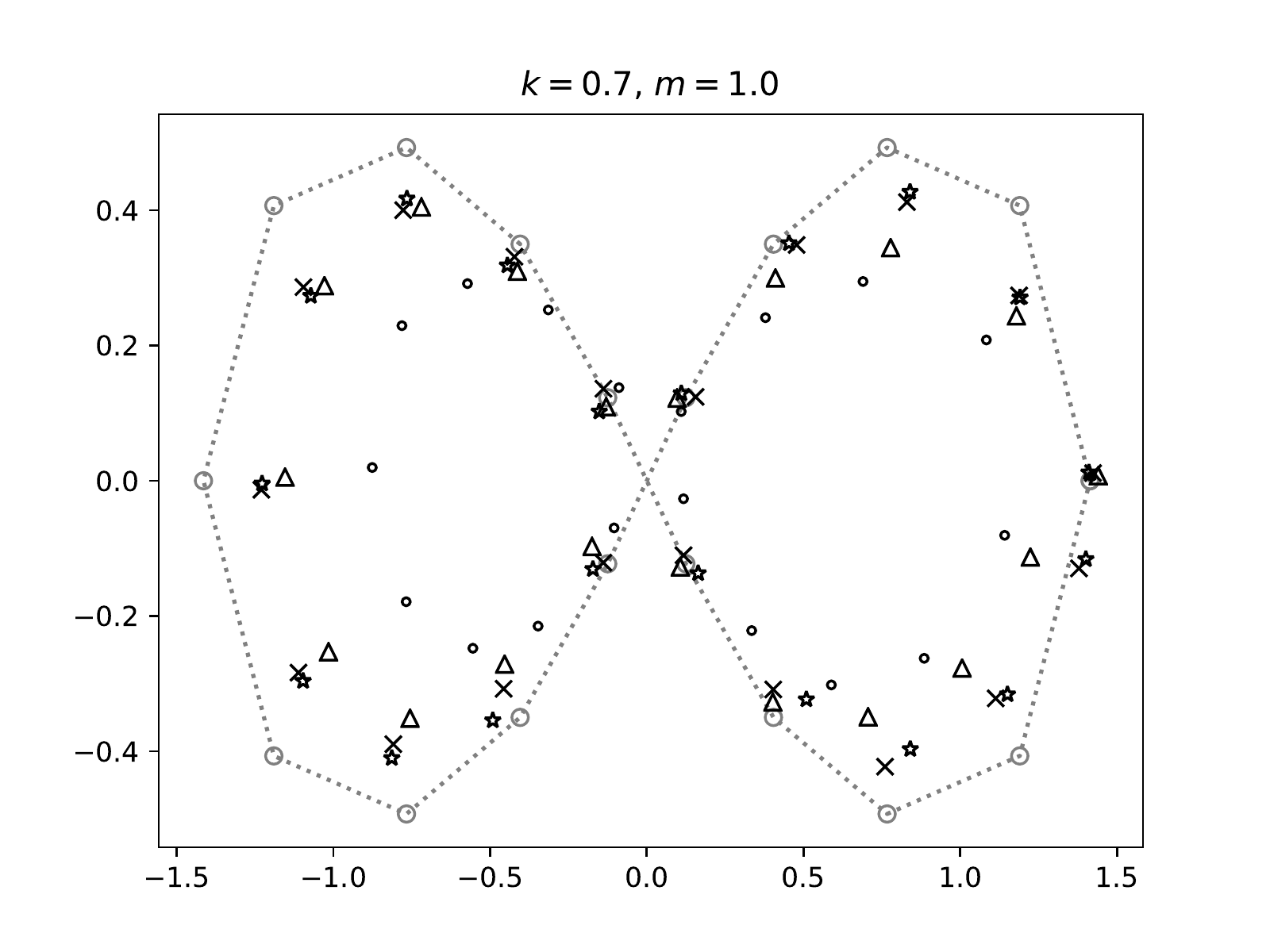}
	\caption{Example trajectories produced by the LQR controllers. We test the LQR policy to follow three types of paths: parabola, circle, and lemniscate. 
	We first train a \metacontroller, then test it on systems with $m=0.1, k=0.7$ (left column), and $m=1.0, k=0.7$ (right column). 
		Dashed line with circles: target trajectories.
		$\star$: optimal policy.
		$\circ$: \metacontroller~trained on $1$ randomly drawn environment.
		$\triangle$: \metacontroller~trained on $3$ randomly drawn environments.
		$\times$: \metacontroller~trained on $10$ randomly drawn environments.
%		We train our \metacontroller for 1 ($\odot$), 3 ($\triangle$), and 10 ($\times$) iterations on randomly drawn contexts. The target trajectories are shown in dashed line with circles. The trajectory given by the optimal controller is represented by hallo.
		}
	\label{fig:exp}
%\vspace{-0.5cm}
\end{figure*}

\section{Conclusion}
\label{sec:con}
%\simon{To incorporate}
In this paper, we give a provably efficient algorithm for learning LQR with contexts.
%The answer to this question has both theoretical and practical importance.
Our result bridges two major fields, learning with contexts and continuous control from a theoretically-principled view.
%\simon{learning with contexts may not be well-defined. I want to refer to contextual bandits...}
%Therefore, the techniques from both fields can be connected and combined and lead to future research.
For future work, it is interesting to study more complex settings, include non-linear control.
Another interesting direction is to design provable algorithm in our setting with safety guarantees~\citep{dann2018policy}.
%From practical point of view, the answer to this question gives principles to many real world problems involving continuous control and contextual information.
%Furthermore, theoretical understanding of this problem can help us design algorithms that enjoy safety guarantees which are required in many real world scenarios, like self-driving~\citep{dann2018policy}.
%\simon{add more}

\newpage
\bibliography{simonduref}
\bibliographystyle{plainnat}

\newpage
\appendix

\section{Proof of Main Results}
\label{sec:proofmain}
This sections devotes to proving the main results.
Before we prove  Proposition~\ref{prop:main}, let us use it to prove Theorem~\ref{thm:main}.
\begin{proof}[Proof of Theorem~\ref{thm:main}]
	We rewrite the Equation~\eqref{reg:pac} as follows.
	\begin{align*}
		\EE_{C,D}&\EE_{\wt{\pi}}
		\Big[J^{\wt{\pi}}_1\big(M_{\Theta_*, C, D},~ x_1\big)\Big]
		- \EE_{C,D}\Big[J^{*}_1\big(M_{\Theta_*, C, D},~  x_1\big)\Big] \\
		&=\frac{1}{K}\sum_{k=1}^{K}\EE_{C,D}\Big[ J_1^{\pi^k}\big(M_{\Theta_*, C, D},~ x_1\big)\Big] - \EE_{C,D}\Big[J_1^{*}\big(M_{\Theta_*, C, D},~ x_1\big)\Big]\\
		%&=\frac{1}{K}\sum_{k=1}^{K}\EE_{C,D}\Big[ J_1^{\pi^k}\big(M_{\Theta_*, C, D},~ x_1\big)\Big] - \EE_{C,D}\Big[J_1^{*}\big(M_{\Theta_*, C, D},~ x_1\big)\Big]\\
		&=\frac{1}{K}\sum_{k=1}^{K}\Big(\EE_{C,D}\Big[ J_1^{\pi^k}\big(M_{\Theta_*, C, D},~ x_1\big)\Big] -  J_1^{\pi^k}\big(M_{\Theta_*, C^{(k)}, D^{(k)}},~ x_1\big) + J_1^{\pi^k}\big(M_{\Theta_*, C^{(k)}, D^{(k)}},~ x_1\big)\\
		&\qquad-J_1^{*}\big(M_{\Theta_*, C^{(k)}, D^{(k)}},~ x_1\big) + J_1^{*}\big(M_{\Theta_*, C^{(k)}, D^{(k)}},~ x_1\big) -\EE_{C,D}\Big[J_1^{*}\big(M_{\Theta_*, C, D},~ x_1\big)\Big]
		\Big)\\
		&=
		R_1 + R_2 + R_3
	\end{align*}
	where 
	\[
	R_1 = \frac{1}{K}\sum_{k=1}^{K}\Big(\EE_{C,D}\big[ J_1^{\pi^k}\big(M_{\Theta_*, C, D},~ x_1\big)\Big]
	-  J_1^{\pi^k}\big(M_{\Theta_*, C^{(k)}, D^{(k)}},~ x_1\big)\Big),
	\]
	\[
	R_2 = \frac{1}{K}\sum_{k=1}^{K}\Big(
	J_1^{*}\big(M_{\Theta_*, C^{(k)}, D^{(k)}},~ x_1\big)
	- \EE_{C,D}\big[ J_1^{*}\big(M_{\Theta_*, C, D},~ x_1\big)\Big]
	\Big),
	\]
	and
	\[
	R_3 = \frac{1}{K}\sum_{k=1}^{K}\Big(
	J_1^{\pi^k}\big(M_{\Theta_*, C^{(k)}, D^{(k)}},~ x_1\big)
	-  J_1^{*}\big(M_{\Theta_*, C^{(k)}, D^{(k)}},~ x_1\big)
	\Big).
	\]
	Let $\cF_{k}$ be the filtration of fixing all randomness before episode $k$. 
	We have $R_1$ and $R_2$ are Martingale difference sum.
	Note that the magnitude of each summand in $R_1$ or $R_2$ is upper bounded by (proved in Lemma~\ref{lem:jtbound} and \ref{lem:jhbound}),
	\[
	H c_q c_x
	\]
	almost surely.
	Therefore, by Azuma's inequality (Theorem~\ref{thm:azum}), we have, with probability greater than $1-\delta/2$,
	\[
	|R_1|+|R_2|\le 2Hc_qc_x\cdot\sqrt{\frac{2\log(8/\delta)}{K}}.
	\]
	Moreover, by Proposition~\ref{prop:main},
	we have with probability greater than $1-\delta/2$,
	\[
	|R_3| \le c\cdot d^{1/2}p\cdot \log^{3/2}(dKHc^2_x\delta^{-1})\cdot\sqrt{\frac{H}{K}},
	\]
	where $c$ is constant depending only polynomially on $H$, $c_q, c_x, c_M$, and $c_w$.
	Combining the above two inequalities, and setting $K$ appropriately, we complete the proof of Theorem~\ref{thm:main}.
\end{proof}
\subsection{Useful Concentration Bounds}
Before we prove the main proposition, we first recall some useful concentration bounds.
\begin{thm}[Azuma's inequality]
	\label{thm:azum}
	Assume that $\{X_s\}_{s\ge0}$ is a martingale and $|X_s-X_{s-1}|\le c_s$ 
	almost surely. Then for all $t > 0$ and all $\epsilon > 0$,
	\[
	\Pr\big[|X_t - X_0|\ge\epsilon\big]
	\le 2\exp\Bigg(\frac{-\epsilon^2}{2\sum_{s=1}^tc_s^2}\Bigg).
	\]
\end{thm}
\begin{thm}[Martingale Concentration, Theorem~16 of \cite{Abbasi-Yadkori2011}]
	\label{thm:matringale-mat}
	Let $\def\cF{\mathcal{F}}$${\cF_t; t\ge 0}$ be a filtration, $(z_t; t\ge 0)$ be an $\def\RR{\mathbb{R}}\RR^d$-valued stochastic process adapted to $(\cF_t)$. Let $(\eta_t; t\ge 1)$ be a real-valued martingale difference process adapted to $\cF_t$. Assume that $\eta_t$ is conditionally sub-Gaussian with constant $L$, i.e., 
	$$
	\forall \gamma>0\quad: \mathbb{E}[\gamma\eta_t|\cF_t]\le \exp(\gamma^2L^2/2).
	$$
	Consider the following martingale 
	$$
	S_t=\sum_{\tau=1}^{t}\eta_\tau z_{t-1}
	$$
	and the matrix-valued processes
	$$
	V_t=I + \sum_{\tau=0}^tz_{t-1}z_{t-1}^\top.
	$$
	Then for any $\delta\in (0,1)$, with probability at least $1-\delta$,
	$$
	\forall t\ge 0,\quad \|S_t\|^2_{V_t^{-1}}\le 2L^2\log\bigg(\frac{\det(V_t)^{1/2}}{\delta}\bigg)
	$$
	where $\|S_t\|_{V_t^{-1}}^2:=S_t^\top V_t^{-1} S_t$. 
\end{thm}
\subsection{Proof of Proposition~\ref{prop:main}}
In this section, we prove the main proposition. We first bound $\det(V^{(k)})$ for any $k$.
\begin{lem}
	\label{lem:detbound}
	For all $k\in [K]$,
	\[
	\det(V^{(k)})
	\le
	\big(1+kHc^2_x/p\big)^{p}.
	\]
\end{lem}
\begin{proof}
	Since $V^{(k)}$ is PD, we have,
	\[
	\det(V^{(k)})\le \Big(\mathrm{tr}(V^{(k)})/p\Big)^{p}\le
	\bigg(1+\sum_{k'=1}^{k}\sum_{h=1}^{H-1}\big\|z^{(k')}_h\big\|_2^2/p\bigg)^{p}.
	\]
	By Assumption~\ref{asmp:regularity}, we have $\|z^{(k')}_h\|_2^2\le c^2_x$.
	This completes the proof.
\end{proof}
Let us then define an event $E_k$ as follows.
\begin{defn}[Good Event]
	We define event $E_k$ as $\{\forall k'\le k:~\Theta_*\in \cC^{(k')}\}$.  
\end{defn}
We then show that the event $E_k$ happens with high probability.
\begin{lem}
	For all $k\in [K]$, we have $\Pr[E_k]\ge 1-\delta$.
\end{lem}
\begin{proof}
	Now we consider $\Theta_* - \Theta^{(k)}$. We immediately have
	\begin{align*}
		\Theta_*^\top - \Theta^{(k)\top} &= \Theta_*^\top - (V^{(k)})^{-1}\bigg(\sum_{k'=1}^{k}\sum_{h=1}^{H-1}z^{(k')}_h (\Theta_*z^{(k')}_h + w_{h+1}^{(k')})^\top\bigg)\\
		&=\bigg(I - (V^{(k)})^{-1}\sum_{k'=1}^k\sum_{h=1}^{H-1}z^{(k')}_hz^{(k')\top}_h\bigg) \Theta_*^\top + (V^{(k)})^{-1}\sum_{k'=1}^k\sum_{h=1}^{H-1}z^{(k')}_hw_{h+1}^{(k')\top}.
	\end{align*}
	Next, we have
	\begin{align*}
		(\Theta_*&-\Theta^{(k)}) V^{(k)}(\Theta_* - \Theta^{(k)})^\top\\
		=&\Theta_*\bigg(I - (V^{(k)})^{-1}\sum_{k'=1}^k\sum_{h=1}^{H-1}z^{(k')}_hz^{(k')\top}_h\bigg)^\top V^{(k)} \bigg(I - (V^{(k)})^{-1}\sum_{k'=1}^k\sum_{h=1}^{H-1}z^{(k')}_hz^{(k')\top}_h\bigg)\Theta_*^{\top}\\
		&+\Theta_*\bigg(I - (V^{(k)})^{-1}\sum_{k'=1}^k\sum_{h=1}^{H-1}z^{(k')}_hz^{(k')\top}_h\bigg)^\top\sum_{k'=1}^k\sum_{h=1}^{H-1}z^{(k')}_hw_{h+1}^{(k')\top}\\
		&+\sum_{k'=1}^k\sum_{h=1}^{H-1}w_{h+1}^{(k')}z^{(k')\top}_h\bigg(I - (V^{(k)})^{-1}\sum_{k'=1}^k\sum_{h=1}^{H-1}z^{(k')}_hz^{(k')\top}_h\bigg)\Theta_*^{\top}\\
		&+\sum_{k'=1}^k\sum_{h=1}^{H-1}w_{h+1}^{(k')}z^{(k')\top}_h (V^{(k)})^{-1}
		\sum_{k'=1}^k\sum_{h=1}^{H-1}z^{(k')}_hw_{h+1}^{(k')\top}
	\end{align*}
	Note that $\sum_{k'=1}^k\sum_{h=1}^{H-1}z^{(k')}_hz^{(k')\top}_h = V^{(k)}-I$ and thus $(V^{(k)})^{-1}\sum_{k'=1}^k\sum_{h=1}^{H-1}z^{(k')}_hz^{(k')\top}_h=I-(V^{(k)})^{-1}$. Hence we have
	\begin{align*}
		\mathrm{tr}\big[(&\Theta_*-\Theta^{(k)}) V^{(k)}(\Theta_* - \Theta^{(k)})^\top\big]\\
		&=\|\Theta^*\|_{(V^{(k)})^{-1}}^2 + 2\mathrm{tr}\Big(\Theta_*(V^{(k)})^{-1}\sum_{k'=1}^k\sum_{h=1}^{H-1}z^{(k')}_hw_{h+1}^{(k')\top}\Big) 
		+ \Big\|\sum_{k'=1}^k\sum_{h=1}^{H-1}z^{(k')}_hw_{h+1}^{(k')\top}\Big\|_{(V^{(k)})^{-1}}^2\\
		&\le \|\Theta_*\|_{(V^{(k)})^{-1}}^2 + 2\Big\|\Theta_{*}\Big\|_{(V^{(k)})^{-1}}\Big\|\sum_{k'=1}^k\sum_{h=1}^{H-1}z^{(k')}_hw_{h+1}^{(k')\top}\Big\|_{(V^{(k)})^{-1}} 
		+ \Big\|\sum_{k'=1}^k\sum_{h=1}^{H-1}z^{(k')}_hw_{h+1}^{(k')\top}\Big\|_{(V^{(k)})^{-1}}^2\\
	\end{align*}
	where $\|X\|_{V}^2:=\mathrm{tr}\big(X^\top VX\big)$ and the last inequality uses Cauchy-Schwartz inequality. Notice that
	$$
	\|\Theta_*\|_{(V^{(k)})^{-1}}\le \|\Theta_*\|_F.
	$$
	Moreover, we have
	\begin{align*}
		\Big\|\sum_{k'=1}^k\sum_{h=1}^{H-1}z^{(k')}_hw_{h+1}^{(k')\top}\Big\|_{(V^{(k)})^{-1}}^2
		&= \bigg\|(V^{(k)})^{-1/2}\sum_{k'=1}^k\sum_{h=1}^{H-1}z^{(k')}_hw_{h+1}^{(k')\top}\bigg\|_F^2 \\
		&= \sum_{j\in[d]}
		\bigg\|(V^{(k)})^{-1/2}\sum_{k'=1}^k\sum_{h=1}^{H-1}w_{h+1,j}^{(k')}z^{(k')}_h\bigg\|_2^2\\
		&=\sum_{j\in[d]} \Big\|\sum_{k'=1}^k\sum_{h=1}^{H-1}w_{h+1,j}^{(k')}z^{(k')}_h\Big\|_{(V^{(k)})^{-1}}^2
	\end{align*}
	By Theorem \ref{thm:matringale-mat}, we have, for every $j\in[d]$, with probability at least $1-\delta/d$, we have,
	$$
	\Big\|\sum_{k'=1}^k\sum_{h=1}^{H-1}w_{h+1,j}^{(k')}z^{(k')}_h\Big\|_{(V^{(k)})^{-1}}^2 \le 2c_w^2\log(d\det(V^{(k)})^{1/2}/\delta).
	$$
	%Next, we use the following lemma to bound $\det(V^{(k)})$. 
	By an union bound, we have, with probability at least $1-\delta$,
	$$
	\Big\|\sum_{k'=1}^k\sum_{h=1}^{H-1}w_{h+1}^{(k')}z^{(k')}_h\Big\|_{(V^{(k)})^{-1}}^2 \le 2dc_w^2\log(d\det(V^{(k)})^{1/2}/\delta).
	$$
	Plugging to $\mathrm{tr}\big[(\Theta_*-\Theta^{(k)}) V^{(k)}(\Theta_* - \Theta^{(k)})^\top\big]$, we have, with probability at least $1-\delta$, 
	\begin{align*}
		\mathrm{tr}\big[(\Theta_*-\Theta^{(k)}) V^{(k)}(\Theta_* - \Theta^{(k)})^\top\big]
		&\le \Big(c_{\Theta}+c_w\sqrt{2d\log(d\det(V^{(k)})^{1/2}/\delta)}\Big)^2\\
		&\le  \Big(c_{\Theta}
		+c_w\sqrt{2d\big(\log d + p\log(1+kHc^2_x/p)/2 + \log\delta^{-1}\big)}\Big)^2.
	\end{align*}
	This completes the proof.
\end{proof}

We define $\one_{E_K}$ as the indicator for $E_K$ happens. 
We denote 
\[
M_*^{(k)} = M_{\Theta_*, C^{(k)}, D^{(k)}},\quad
M^{(k)} = 
M_{\Theta^{(k)}, C^{(k)}, D^{(k)}}, \quad\text{and}
\quad
y^{(k)}_h=[x_h^{(k)\top}, u_h^{(k)\top}]^\top.
\]
On $E_k$, we have
$$
\forall k\in [K]:\quad J_1^*(\widetilde{M}^{(k)}, x_1^{k})\le J_1^*({M}_*^{(k)}, x_1^{k}).
$$
We denote $\Delta^{(k)}:=J_h^{\pi^{k}}(M_*^{(k)}, x_1) - J_h^{*}(M_*^{(k)}, x_1)$. We can rewrite  $\eqref{eq:reg}$ as
$$
\mathrm{Regret}(KH) = \sum_{k=1}^K\one_{E_k}\Delta^{(k)}+ \sum_{k=1}^K(1-\one_{E_k})\Delta^{(k)},
$$
where the second term is non-zero with probability less than $\delta$. For the first term, we have
$$
\one_{E_k}\Delta^{(k)}\le \one_{E_k}\big[J_1^{\pi^k}(M_*^{(k)}, x_1) - J_1^*(\widetilde{M}^{(k)}, x_1))\big]=:\one_{E_k}\cdot\widetilde{\Delta}^{(k)}_1,
$$
where
$$
\widetilde{\Delta}^{(k)}_h=J_h^{\pi^k}(M_*^{(k)}, x_h) -J^*_h(\widetilde{M}^{(k)}, x_h).
$$
Let us consider $\widetilde{\Delta}_h^{(k)}$. We denote filtration $\cF_{k,h}$ as fixing the trajectory  up to time $(k,h)$ and all $\{C^{(k')}, D^{(k')}\}_{k'\le k}$.

We have
\begin{align*}
	\widetilde{\Delta}_h^{(k)}=& x^{(k)\top}_hQ_hx_h^{(k)} + u^{(k)\top}_hR_hu_h^{(k)}  +  \mathbb{E}_{w_{h+1}^{(k)}}[J_{h+1}^{\pi^{k}}(M_*^{(k)}, x_{h+1}^{(k)})~|~\cF_{k,h}] \\
	&- x^{(k)\top}_{h}Q_hx_{h}^{(k)} - u^{(k)\top}_{h}R_hu_{h}^{(k)}\\
	&-
	\mathbb{E}_{w_{h+1}^{(k)}}\Big[(\widetilde{M}^{(k)}z^{(k)}_{h}+w^{(k)}_{h+1})^\top P_{h+1}(\widetilde{M}^{(k)})(\widetilde{M}^{(k)}z^{(k)}_{h}+w^{(k)}_{h+1})~|~\cF_{k,h}\Big]\\
	&-C_{h+1}(\widetilde{M}^{(k)})\\
	=&\mathbb{E}_{w_{h+1}^{(k)}}[J_{h+1}^{\pi^{k}}(M_*^{(k)}, x_{h+1}^{(k)})~|~\cF_{k,h}] \\
	&-
	\mathbb{E}_{w_{h+1}^{(k)}}\Big[(\widetilde{M}^{(k)}z^{(k)}_{h}+w^{(k)}_{h+1})^\top P_{h+1}(\widetilde{M}^{(k)})(\widetilde{M}^{(k)}z^{(k)}_{h}+w^{(k)}_{h+1})~|~\cF_{k,h}\Big]\\
	&-C_{h+1}(\widetilde{M}^{(k)})\\
	=& \mathbb{E}_{w_{h+1}^{(k)}}[J_{h+1}^{\pi^{k}}(M_*^{(k)}, x_{h+1}^{(k)})~|~\cF_{k,h}] 
	- J_{h+1}^{\pi^{k}}(M_*^{(k)}, x_{h+1}^{(k)}) + J_{h+1}^{\pi^{k}}(M_*^{(k)}, x_{h+1}^{(k)})\\
	&-\big(\widetilde{M}^{(k)}z^{(k)}_{h}\big)^\top P_{h+1}(\widetilde{M}^{(k)}) \big(\widetilde{M}^{(k)}z^{(k)}_{h}\big) - C_{h+1}(\widetilde{M}^{(k)}) \\
	&-\mathbb{E}_{w_{h+1}^{(k)}}\Big[\big(w^{(k)}_{h+1}\big)^\top P_{h+1}(\widetilde{M}^{(k)})w^{(k)}_{h+1}~\big|~\cF_{k,h}\Big]\\
	=&\delta_h^{(k)}+ J_{h+1}^{\pi^{k}}(M_*^{(k)}, x_{h+1}^{(k)})-\big(\widetilde{M}^{(k)}z^{(k)}_{h}\big)^\top P_{h+1}(\widetilde{M}^{(k)}) \big(\widetilde{M}^{(k)}z^{(k)}_{h}\big) - C_{h+1}(\widetilde{M}^{(k)})\\
	&-\mathbb{E}_{w_{h+1}^{(k)}}\Big[\big(x_{h+1}^{(k)}-M_*^{(k)}z^{(k)}_h\big)^\top P_{h+1}(\widetilde{M}^{(k)})\big(x_{h+1}^{(k)}-M_*^{(k)}z^{(k)}_h\big)~\big|~\cF_{k,h}\Big]\\
	=&\delta_h^{(k)}+ J_{h+1}^{\pi^{k}}(M_*^{(k)}, x_{h+1}^{(k)})-\big(\widetilde{M}^{(k)}z^{(k)}_{h}\big)^\top P_{h+1}(\widetilde{M}^{(k)}) \big(\widetilde{M}^{(k)}z^{(k)}_{h}\big) - C_{h+1}(\widetilde{M}^{(k)})\\
	&-\mathbb{E}_{w_{h+1}^{(k)}}\Big[\big(x_{h+1}^{(k)}\big)^\top P_{h+1}(\widetilde{M}^{(k)})\big(x_{h+1}^{(k)}\big)~\big|~\cF_{k,h}\Big]
	+\big(M_*^{(k)}z^{(k)}_h\big)^\top P_{h+1}(\widetilde{M}^{(k)})\big(M_*^{(k)}z^{(k)}_h\big)\\
	=&\delta_h^{(k)} + \delta_h^{'(k)}+\delta_h^{''(k)}+J_{h+1}^{\pi^{k}}(M_*, x_{h+1}^{(k)})-J_{h+1}^{*}(\widetilde{M}^{(k)}, x_{h+1}^{(k)})~~~~~~~~~~~~~~~~~~~~~~~~~~
\end{align*}
where 
\begin{align}\delta_h^{(k)} &= \mathbb{E}_{w_{h+1}^{(k)}}[J_{h+1}^{\pi^{k}}(M_*^{(k)}, x_{h+1}^{(k)})~|~\cF_{k,h}] - J_{h+1}^{\pi^{k}}(M_*^{(k)}, x_{h+1}^{(k)})\\\delta_h^{'(k)}&=\big(x_{h+1}^{(k)}\big)^\top P_{h+1}(\widetilde{M}^{(k)})\big(x_{h+1}^{(k)}\big)-\mathbb{E}_{w_{h+1}^{(k)}}\Big[\big(x_{h+1}^{(k)}\big)^\top P_{h+1}(\widetilde{M}^{(k)})\big(x_{h+1}^{(k)}\big)~\big|~\cF_{k,h}\Big]\\\delta_h^{''(k)}&=\big(M_*^{(k)}z^{(k)}_h\big)^\top P_{h+1}(\widetilde{M}^{(k)})\big(M_*^{(k)}z^{(k)}_h\big)-\big(\widetilde{M}^{(k)}z^{(k)}_{h}\big)^\top P_{h+1}(\widetilde{M}^{(k)}) \big(\widetilde{M}^{(k)}z^{(k)}_{h}\big).\end{align}
By induction, we have
$$
\sum_{k'=1}^{k}\widetilde{\Delta}_1^{(k)}\le\sum_{k'=1}^{k}\sum_{h=1}^{H-1}\big(\delta_h^{(k)}+\delta_h^{'(k)}+\delta_h^{''(k)}\big).
$$
Notice that $\delta^{(k)}_h$ and $\delta_h^{(k)}$ are Martingale difference adapted to $\cF_{k,h}$. We can well bound the sum of them via Azuma's inequality. 
\begin{lem}
	\label{lem:jtbound}
	For all $h\in [H]$,  $|J_{h}^{\pi^{k}}(M_*^{(k)}, x_{h}^{(k)})|\le (H-h+1)\cdot c_q\cdot c_x.$
\end{lem}
\begin{proof} Prove by induction on $h$. The base case  $J_{H}^{\pi^{k}}(M_*^{(k)}, x_{H}^{(k)})=x_{H}^{(k)\top} Q_H x_{H}^{(k)}\le c_qc_x^2$ holds straightforwardly.  Consider an arbitrary $h<H$, we have
	$$
	J_{h}^{\pi^{k}}(M_*^{(k)}, x_h^{(k)}) =x^{(k)\top}_hQ_hx_h^{(k)} + u^{(k)\top}_hR_hu_h^{(k)}  +  \mathbb{E}_{w_{h+1}^{(k)}}[J_{h+1}^{\pi^{k}}(M_*^{(k)}, x_{h+1}^{(k)})~|~\cF_{k,h}]\le c_q c_x +(H-h)\cdot c_q\cdot c_x
	$$
	as desired. 
\end{proof}
\begin{lem}
	\label{lem:jhbound}
	For all $x\in X$, we have $\one_{E_K}|J_{h}^{*}(\widetilde{M}^{(k)}, x)|\le c_qc_x$.
\end{lem}
\begin{proof}
	Follows from Assumption~\ref{asmp:regularity}.
\end{proof}
We are now ready to prove Proposition~\ref{prop:main}.
\begin{proof}[Proof of Proposition~\ref{prop:main}]
	Thus by Azuma's inequality, we have, with probability at least $1-\delta$,
	$$
	\Big|\sum_{k'=1}^{k}\sum_{h=1}^{H-1}\delta^{(k)}_h\Big| \le \sqrt{2kH\cdot\big[(H-h+1)qc_x+c_qc_x\big]^2\cdot \log\frac{2}\delta}.
	$$
	And, with probability at least $1-\delta$,
	$$
	\Big|\sum_{k'=1}^{k}\sum_{h=1}^{H-1}\delta^{'(k)}_h\Big| \le \sqrt{8kH\cdot c_x^2c_q^2\cdot \log\frac{2}\delta}.
	$$
	For $\sum \delta^{''(k)}_h$, we bound it here. 
	\begin{align*}\Big|\sum_{k'=1}^{k}\sum_{h=1}^{H-1}\delta^{''(k)}_h\Big|\le& \sum_{k'=1}^{k}\sum_{h=1}^{H-1}\Big|\delta^{''(k)}_h\Big| =\sum_{k'=1}^{k}\sum_{h=1}^{H-1}\Big|\|P_{h+1}(\widetilde{M}^{(k)})^{1/2}\big(\widetilde{M}^{(k)}y^{(k)}_h\big)\|_2^2 - \|P_{h+1}(\widetilde{M}^{(k)})^{1/2}\big({M}^{*}y^{(k)}_h\big)\|_2^2\Big|\\\le& \sum_{k'=1}^{k}\sum_{h=1}^{H-1} \Big|\big(\|P_{h+1}(\widetilde{M}^{(k)})^{1/2}\big(\widetilde{M}^{(k)}y^{(k)}_h\big)\|_2 - \|P_{h+1}(\widetilde{M}^{(k)})^{1/2}\big({M}^{*}y^{(k)}_h\big)\|_2\big)\\\cdot&\big(\|P_{h+1}(\widetilde{M}^{(k)})^{1/2}\big(\widetilde{M}^{(k)}y^{(k)}_h\big)\|_2 + \|P_{h+1}(\widetilde{M}^{(k)})^{1/2}\big({M}^{*}y^{(k)}_h\big)\|_2\big)\Big|\\\le & \Big[\sum_{k'=1}^{k}\sum_{h=1}^{H-1} \Big(\|P_{h+1}(\widetilde{M}^{(k)})^{1/2}\big(\widetilde{M}^{(k)}y^{(k)}_h\big)\|_2 - \|P_{h+1}(\widetilde{M}^{(k)})^{1/2}\big({M}^{*}y^{(k)}_h\big)\|_2\Big)^2 \Big]^{1/2}\\\cdot &\Big[\sum_{k'=1}^{k}\sum_{h=1}^{H-1} \Big(\|P_{h+1}(\widetilde{M}^{(k)})^{1/2}\big(\widetilde{M}^{(k)}y^{(k)}_h\big)\|_2 + \|P_{h+1}(\widetilde{M}^{(k)})^{1/2}\big({M}^{*}y^{(k)}_h\big)\|_2\Big)^2\Big]^{1/2} \end{align*}
	Notice that $\|P_{h+1}(\widetilde{M}^{(k)})^{1/2}\big(\widetilde{M}^{(k)}y^{(k)}_h\big)\|_2\le c_qc_xc_{\Theta}$ and $\|P_{h+1}(\widetilde{M}^{(k)})^{1/2}\big({M}^{*}y^{(k)}_h\big)\|_2\le c_qc_x$.
	Hence
	\begin{align*}\Big[\sum_{k'=1}^{k}\sum_{h=1}^{H-1}& \Big|\Big(\|P_{h+1}(\widetilde{M}^{(k)})^{1/2}\big(\widetilde{M}^{(k)}y^{(k)}_h\big)\|_2 + \|P_{h+1}(\widetilde{M}^{(k)})^{1/2}\big({M}^{*}y^{(k)}_h\big)\|_2\Big)^2\Big]^{1/2} \\&\le \sqrt{kH\cdot (c_qc_x(1+c_{\Theta}))^2}.\end{align*}
	Moreover, by triangle inequality, we have
	\begin{align*}\Big|&\|P_{h+1}(\widetilde{M}^{(k)})^{1/2}\big(\widetilde{M}^{(k)}y^{(k)}_h\big)\|_2 - \|P_{h+1}(\widetilde{M}^{(k)})^{1/2}\big({M}^{*}y^{(k)}_h\big)\|_2\Big|\\&\le \|P_{h+1}(\widetilde{M}^{(k)})^{1/2}\big(\widetilde{M}^{(k)} - {M}^{*}\big)y^{(k)}_h\|_2\\&\le c_q\|\big(\widetilde{M}^{(k)} - {M}^{*}\big)y^{(k)}_h\|_2\\&\le c_q\|\big(\widetilde{M}^{(k)} - {M}^{*}\big)(V^{(k)})^{1/2}(V^{(k)})^{-1/2}y_{h}^{(k)}\|_2\\&\le c_q\|\big(\widetilde{M}^{(k)} - {M}^{*}\big)(V^{(k)})^{1/2}\|_2\|(V^{(k)})^{-1/2}y_{h}^{(k)}\|_2\\&\le c_q\cdot\sqrt{\beta^{(k)}}\cdot \|(V^{(k)})^{-1/2}y_{h}^{(k)}\|_2. \end{align*}
	By Assumption 2, we also have $ \|(V^{(k)})^{-1/2}y_{h}^{(k)}\|_2\le \|(y_{h}^{(k)}\|_2\le \sqrt{c_x}$. Hence,
	\begin{align*}\Big|\|P_{h+1}(\widetilde{M}^{(k)})^{1/2}&\big(\widetilde{M}^{(k)}y^{(k)}_h\big)\|_2 - \|P_{h+1}(\widetilde{M}^{(k)})^{1/2}\big({M}^{*}y^{(k)}_h\big)\|_2\Big|\\&\le c_q\sqrt{c_x}\cdot\sqrt{\beta^{(k)}}\cdot\min\big(\|(V^{(k)})^{-1/2}y_{h}^{(k)}\|_2, 1\big)\end{align*}
	Combining the above equations, we have,
	\begin{align*}\Big|\sum_{k'=1}^{k}\sum_{h=1}^{H-1}\delta^{''(k)}_h\Big|&\le\sqrt{kH\cdot (c_qc_x(1+c_{\Theta}))^2}\cdot c_q\sqrt{c_x}\cdot\sqrt{\beta^{(k)}}\cdot \sqrt{\sum_{k'=1}^{k}\sum_{h=1}^{H-1}\min\big(\|(V^{(k)})^{-1/2}y_{h}^{(k)}\|_2^2, 1\big)}\\&\le2c_x^{3/2}c_q^2c_{\Theta}\cdot \sqrt{\beta^{(k)}}\cdot\sqrt{\sum_{k'=1}^{k}\sum_{h=1}^{H-1}\log\Big(1+\|(V^{(k)})^{-1/2}y_{h}^{(k)}\|_2^2\Big)}\cdot\sqrt{kH}.\end{align*}
	Lastly, by Lemma 8 of \cite{yang2019reinforcement}, we have
	$$
	\sum_{k'=1}^{k}\sum_{h=1}^{H-1}\log\Big(1+\|(V^{(k)})^{-1/2}y_{h}^{(k)}\|_2^2\Big)\le 2H\log\mathrm{det}(V^{(k)}).
	$$
	Together with Lemma \ref{lem:detbound}, we have 
	$$
	2H\log\det(V^{(k)})\le 2Hp\cdot\log\big(1+kHc_x^2/p\big).
	$$
	Overall, we have,
	\begin{align*}\Big|\sum_{k'=1}^{k}\sum_{h=1}^{H-1}\delta^{''(k)}_h\Big|&\le 2c_x^{3/2}c_q^2c_{\Theta}\cdot \sqrt{2Hp\cdot\log\big(1+kHc^2_x/p\big)\cdot (\beta^{(k)})\cdot kH}.\end{align*}
	Putting everything together, with probability at least $1-2\delta$, we have
	\begin{align*}\mathrm{reg}(KH)&\le\sum_{k'=1}^{K}\sum_{h=1}^{H-1}\big(\delta_h^{(k)}+\delta_h^{'(k)}+\delta_h^{''(k)}\big)\\&\le\sqrt{2KH\cdot\big[(H-h+1)c_qc_x+c_qc_x\big]^2\cdot \log\frac{2}\delta}+\sqrt{8KH\cdot c_x^2c_q^2\cdot \log\frac{2}\delta} \\&\quad+2c_x^{3/2}c_q^2c_{\Theta}\cdot \sqrt{2Hp\cdot\log\big(1+KHc^2_x/p\big)\cdot (\beta^{(K)})\cdot KH}\\&\le c_{H}\cdot d^{1/2}p\cdot \log^{3/2}(dKHc^2_x\delta^{-1})\cdot\sqrt{KH},
	\end{align*}
	where $c_H$ is a constant depending on $H, c_q, c_x, c_{\Theta}$ and $c_w$.
\end{proof}
%\begin{enumerate}
%\item $\Theta^{1}_{1}\gets I, V_1\gets \lambda I$ for some $\lambda >0$.
%\item $\cB\gets\{\}$
%\item For $k=1, 2, \ldots, K$:
%\item Compute $\Theta^k$
%\item \begin{enumerate}
%	\item For $h=1, 2, \ldots, H-1$
%	\item \quad Compute $\Theta^{1}_{1}$ by $\argmin_{\Theta\in \cC} J^{\Theta}(x_t)$
%	\item Solve for $u^{k}_{h}$ and play it
%	\item observe the next state $x^k_{h+1}$, add $(x^k_h, u^k_h, x^k_{h+1})$ to $\cB$
%\end{enumerate}
%\item $V_{k+1}\gets V_t + \sum_{h=1}^{H-1} [C_k x^k_h; D_k x^k_h][C_k x^k_h; D_k x^k_h]^{\top}$
%\end{enumerate}
%Now we fill in the details.

\section{Concrete Choice of the Parameters}\label{sec:setting}
We further augment the state so that the first coordinate is a constant with value $1$.
More specifically, we set the state $x_h = [1; z_h; v_h] \in \mathbb{R}^5$.
We set 
\[
Q_h = 
\left(\begin{array}{c c c}
\|z_h^*\|_2^2 & -z_h^* & 0\\
-z_h^* & I & 0\\
0 & 0 & 0\\
\end{array}
\right)
\]
so that for any state $x_h$, $x_h^TQ_hx_H = \|z_h - z_h^*\|_2^2$. We set $R_h = I_2$
We set
\[
\Theta_*= \left( \begin{array}{c c c c c c c}
1 & 0 & 0  & 0 & 0 & 0 &0 \\
0 & 1 & 0  & 1 & 0 & 0 &0 \\
0 & 0 & 1  & 0 & 1 & 0 &0 \\
0 & 0 & 0  & k & 0 & 1 &0 \\
0 & 0 & 0  & 0 & k & 0 &1 \\
\end{array}
\right),
\]
$C$ to be the $5 \times 5$ identity matrix and $D$ to be $I / m$ with size $2 \times 2$ where $m$ is sampled from the uniform distribution over $[0.1, 10]$, 
to represent the physical law in Equation~\ref{equ:law}.

\end{document}